\documentclass{article}
\usepackage{iclr2027_conference,times}

\usepackage{amsmath,amsfonts,bm}

\def\eqref#1{equation~\ref{#1}}

\def\1{\bm{1}}

\DeclareMathAlphabet{\mathsfit}{\encodingdefault}{\sfdefault}{m}{sl}
\SetMathAlphabet{\mathsfit}{bold}{\encodingdefault}{\sfdefault}{bx}{n}

\DeclareMathOperator*{\argmax}{arg\,max}

\usepackage{xstring}
\newcommand{\papersubmission}{preprint}   
\newcommand{\paperheader}{Under review as a conference paper at ICLR 2027}
\IfStrEq{\papersubmission}{underreview}{%
  \iclrfinalfalse\renewcommand{\paperheader}{Under review as a conference paper at ICLR 2027}}{}%
\IfStrEq{\papersubmission}{preprint}{%
  \iclrfinaltrue\renewcommand{\paperheader}{Preprint}}{}%
\IfStrEq{\papersubmission}{accepted}{%
  \iclrfinaltrue\renewcommand{\paperheader}{Accepted as a conference paper at ICLR 2027}}{}%
\IfStrEq{\papersubmission}{cameraready}{%
  \iclrfinaltrue\renewcommand{\paperheader}{Published as a conference paper at ICLR 2027}}{}%
\usepackage{xcolor}
\definecolor{citeblue}{rgb}{0.20,0.45,0.75}
\usepackage[colorlinks=true, citecolor=citeblue, linkcolor=red, urlcolor=black,hypertexnames=false]{hyperref}
\usepackage{url}
\usepackage{graphicx}
\ifdefined\pdfminorversion
\fi
\usepackage{float}
\usepackage{booktabs}
\usepackage{multirow}
\usepackage{enumitem}
\usepackage{algorithm}
\usepackage{algorithmic}
\usepackage{amsmath,amssymb}
\usepackage{amsthm}
\theoremstyle{definition}
\newtheorem{theorem}{Theorem}
\usepackage{pifont}
\usepackage{tikz}
\usetikzlibrary{positioning,arrows.meta,calc,fit,shapes.geometric,backgrounds}
\usepackage{pgfplots}
\pgfplotsset{compat=1.18}
\usepgfplotslibrary{groupplots}
\usepackage{tabularx}
\usepackage{longtable}
\newcolumntype{Y}{>{\raggedright\arraybackslash}X}
\usepackage{listings}
\newcommand{\takeaway}[1]{\vspace{0.3em}\noindent\textbf{#1}\ }
\newcommand{\up}[1]{\,\textcolor{green!55!black}{\scriptsize+#1}}
\usepackage{xspace}
\newcommand{\mc}{Mara Chain\xspace}

\newcommand{\MCc}{Mara Chain\xspace}

\title{Mara Chain: Rethinking Failure as a Stepping Stone for AI System Auto-Evolution}
\author{%
Yubin Lyu\textsuperscript{1}\thanks{Correspondence: \texttt{jacklv@pku.org.cn}, \texttt{jiawei.fei@kaust.edu.sa}. Code: \url{https://github.com/ant-research/AntOmniEvo}.}%
\And Fu Li\textsuperscript{2}%
\And Jiawei Fei\textsuperscript{3}\footnotemark[1]%
\And Yang Zhao\textsuperscript{3}%
\And Weixing Mei\textsuperscript{1}%
\And Yinan Wu\textsuperscript{1}%
\AND
\textsuperscript{1}Ant Group \quad \textsuperscript{2}Beijing Intelligent Game and Decision Lab \quad \textsuperscript{3}Beijing Defense Innovation Institute%
}

\begin{document}
\maketitle
\lhead{\paperheader}

\begin{abstract}

Optimizing deployed AI systems increasingly amounts to editing prompts, skills,
harnesses, and code rather than model weights. Existing approaches commonly
optimize these artifacts through propose--evaluate--select procedures, where
candidate configurations are evaluated and only those meeting an acceptance
criterion are selected. Yet our analysis shows that discarded candidates often
contain information critical for subsequent optimization. Discarding them causes
later proposals to revisit the same failure modes. We introduce \emph{Mara
Chain}, a refinement procedure that turns rejected candidates into stepping
stones. Rather than discarding a rejected candidate, Mara Chain retains and
iteratively refines it using evidence accumulated across preceding attempts.
The procedure limits each refinement chain to a fixed depth and applies
Pareto-filtered Top-$N$ selection to bound the candidate pool. Across AppWorld
skill optimization, TerminalBench~2.1 harness optimization, and MuSiQue
retrieval-pipeline optimization, Mara Chain delivers greater task-performance
gains with fewer rollouts.
It outperforms GEPA, ACE, and SkillOpt-Lite by up to $20.5\,\%$ in relative
performance on AppWorld, reaching the target score with $65.5\,\%$ fewer
rollouts than GEPA. It improves the pass rate by $20.2$ and $22.5$ percentage
points over AHE and Meta-Harness on TerminalBench~2.1, respectively, and
improves MuSiQue test nDCG@10 and Recall@10 by $0.104$ and $0.131$ over a
hand-written retrieval pipeline.

\end{abstract}
\vspace{-1.2em}
\begin{figure}[H]
\centering
\begin{minipage}[t]{0.48\textwidth}
  \centering
  \includegraphics[width=\linewidth]{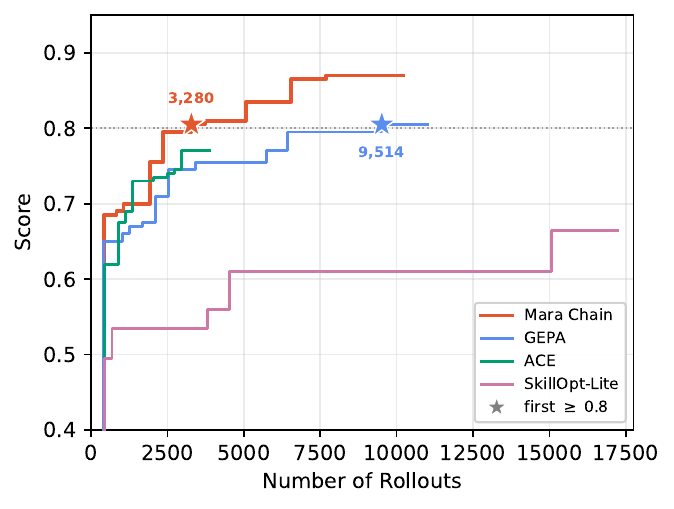}
  \par\vspace{-0.4em}{\small (a) AppWorld, GLM-5}
\end{minipage}\hfill
\begin{minipage}[t]{0.48\textwidth}
  \centering
  \includegraphics[width=\linewidth]{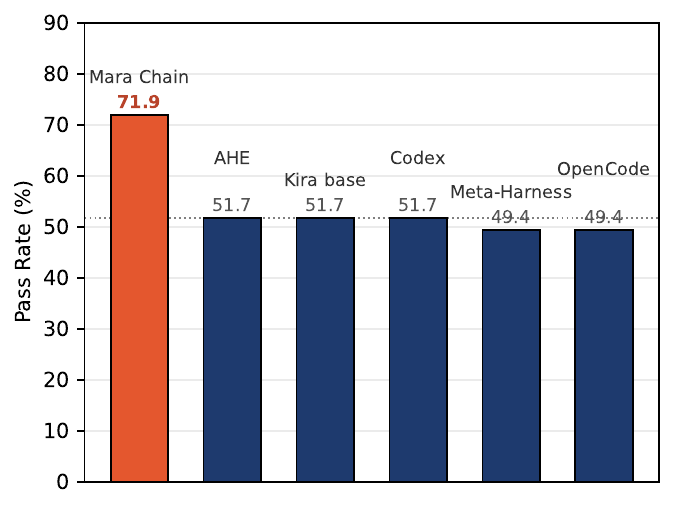}
  \par\vspace{-0.4em}{\small (b) TerminalBench-2.1, GLM-5}
\end{minipage}
\caption{Comparison on AppWorld and TerminalBench-2.1. On AppWorld, Mara Chain
reaches the target validation score using $65.5\,\%$ fewer rollouts than GEPA
and achieves a higher final score, while ACE and SkillOpt-Lite do not reach the
target score. On TerminalBench-2.1, Mara Chain achieves the highest pass rate
among the compared methods.}
\label{fig:main-results}
\end{figure}

\section{Introduction}
\label{sec:intro}

AI systems are increasingly optimized through mutable prompts, skills, harnesses, 
scripts, configurations, and code rather than model-weight updates. A standard
approach is to place these artifacts in a propose--evaluate--select loop, where
an LLM proposes a candidate configuration and an evaluator decides whether to
select it. Such procedures have been applied to prompts
\citep{yang2024opro,zhou2023ape,fernando2024promptbreeder,guo2024evoprompt,przyzant2023protegi},
compound systems \citep{khattab2024dspy,cheng2024trace,wu2025optimas,zhang2024aflow},
code \citep{novikov2025alphaevolve}, and agent skills and harnesses
\citep{yang2026skillopt,shen2026skilloptlite,lee2026metaharness,lin2026ahe}.
Unlike weight fine-tuning, these artifact-optimizing methods can adapt a deployed system
without changing its underlying model.

However, existing methods often suffer from \emph{persistent failure barriers}: 
they repeatedly re-analyze and re-attempt the same failing task, 
making no real progress even as the rollout budget grows.
Inspecting the rollout traces, we find that the culprit is how these methods handle candidates: 
any candidate that does not improve the score on the held-out set is discarded outright. 
Yet such a discarded candidate is rarely worthless. 
It may carry a partial fix that can be extended, 
implement a sound direction whose implementation contains a repairable bug, 
or be a wrong attempt whose measured effect rules out one route to the goal,
prompting later proposals to try a different one. By throwing these candidates away, 
existing methods discard the very evidence needed to move past the barrier, 
forcing each subsequent proposal to reconstruct the same diagnosis from scratch.
We argue that these discarded candidates are an underused resource: 
retaining and building on them, rather than throwing them away, 
is key to breaking the barrier and can substantially improve 
the performance of artifact optimization.

We introduce the \mc\footnote{The name alludes to M\={a}ra, the embodiment of
the obstacles that, in Buddhist tradition, confront the
practitioner on the path to awakening. In the same spirit, the method advances
through repeated setbacks: each attempt that fails to clear the acceptance bar
is retained rather than discarded, its failed experience accumulated as a
stepping stone, until the accumulated attempts jointly reach the goal.}, 
a procedure that, within the propose--evaluate--select loop for compound AI system optimization, 
refines rejected candidates instead of discarding them, 
forming a history-conditioned chain to exploit the evidence in failed attempts.
In each epoch, when a candidate does not clear the configured acceptance bar, 
Mara Chain retains its rollout traces, residual failures, structured analysis, and change history.
It then generates a sequence of descendants: at each step, the LLM analyzes the accumulated evidence 
and proposes the next candidate accordingly. 
To keep this bounded, the chain runs for only a fixed depth; 
if it produces no descendant that beats the initial candidate, 
the chain is discarded, so a failed chain adds no lasting cost.
Furthermore, to avoid an ever-growing candidate pool,
in which many candidates remain mutually non-dominated and cannot be pruned,
the candidate pool applies a top-$N$ truncation to
the Pareto-filtered set, keeping the $N$ candidates with the highest scores
on the validation set.
As a result, Mara Chain overcomes the stalling that traps existing methods while keeping its cost bounded.
We evaluate Mara Chain on three benchmarks---AppWorld~\citep{trivedi2024appworld},
TerminalBench~2.1~\citep{merrill2026terminalbench}, and MuSiQue~\citep{trivedi2022musique}---each
optimizing a distinct class of tunable artifacts, across three LLM models:
GLM-5~\citep{zai2026glm5}, DeepSeek-V4-Pro, and Qwen3.5-397B-A17B.
Figure~\ref{fig:main-results} presents the main comparisons on AppWorld and
TerminalBench~2.1.
On \textbf{AppWorld}, whose artifacts form a \emph{skill} configuration, Mara Chain
improves over the strongest skill optimizers---GEPA~\citep{agrawal2025gepa},
ACE~\citep{zhang2025ace}, and SkillOpt-Lite~\citep{shen2026skilloptlite}---by up to $20.5$\,\%,
reaches the target validation score using $65.5$\,\% fewer rollouts than GEPA,
and achieves a higher final score; ACE and SkillOpt-Lite do not reach the target score.
On \textbf{TerminalBench~2.1}, an \emph{agent-harness} configuration, the optimized harness
improves the pass rate by $39.1$\,\% over the harness optimizers
AHE~\citep{lin2026ahe} and Meta-Harness~\citep{lee2026metaharness}.
On \textbf{MuSiQue}, a \emph{retrieval-pipeline} configuration, Mara Chain improves test
nDCG@10 / Recall@10 by $34.6$\,\% / $39.8$\,\% over the hand-written default pipeline,
showing the approach is not specific to any single artifact type.

\section{Problem Statement}
\label{sec:problem}

\paragraph{Tunable artifacts.}
Tunable artifacts are the modifiable behavioral components of a system
under optimization, including prompts, scripts, harnesses, configuration files, and source
code. We represent these artifacts as a concrete directory, whether
the target system is an AI agent, a workflow, or a standalone algorithm. Rather
than optimizing a constrained parameter vector or a domain-specific language,
the optimizer directly adds, removes, and modifies files in this directory. For
each target system, the user specifies an \emph{artifact schema} comprising the
directory structure and the semantic role of each file. This schema establishes
the artifact mapping and defines the set of files that the proposer may modify
(App.~\ref{app:framework}). Consequently, any system with tunable artifacts
that can be mapped in this way and evaluated reproducibly---such as a skill,
evaluation harness, retrieval pipeline, or standalone computational kernel---can be
optimized by the same procedure.

\paragraph{Optimization objective.}
Given an artifact configuration $\sigma$ and a task distribution $\mathcal{D}$,
we seek
\begin{equation}
\label{eq:objective}
\sigma^\star \;=\; \arg\max_{\sigma}\;
\mathbb{E}_{\tau\sim\mathcal{D}}\!\left[\,\text{score}\bigl(\text{exec}(\sigma,\tau)\bigr)\,\right]
\quad\text{s.t.}\quad
\lvert\{\text{rollouts}\}\rvert \le B,
\end{equation}
where $\text{exec}(\sigma,\tau)$ executes the system instantiated with artifact
configuration $\sigma$ on task $\tau$, and $\text{score}\in[0,1]$ is the
per-instance performance measure.
Because each $\text{exec}(\sigma,\tau)$ is a costly rollout, this is an
optimization problem under a fixed rollout budget $B$: the goal is to attain
as high an expected score as possible while issuing a limited number of rollouts.


\paragraph{Persistent failure barriers.}
A task $\tau$ imposes a set of conditions. Executing $\sigma$ on $\tau$ yields a
\emph{residual} $\delta(\sigma,\tau)$: the unsatisfied conditions exposed by the
execution. Across successive candidate configurations, this residual can stay
invariant---unchanged no matter how much rollout budget is spent---rather than
being a merely hard condition that additional candidates or rollouts would
eventually resolve. We call this phenomenon a \emph{persistent failure barrier}:
a condition whose residual remains invariant across iterations and prevents
$\tau$ from succeeding.
Such barriers commonly arise under the standard propose--evaluate--select search
even when a correct modification is reachable, especially on long-horizon tasks.
As a result, the correct modification is repeatedly discarded and the same
residual is re-diagnosed across iterations. Many factors drive this, such as a
partial or well-directed change that yields no immediate score gain despite
being on the right track, a correct modification that is never exercised at execution,
or a target condition gated by unmet prerequisites. Regardless of the cause,
additional rollouts leave the residual unchanged.

\section{\MCc}
\label{sec:rc}

We introduce Mara Chain, an improved propose--evaluate--select procedure that
searches for a high-performing, task-specific artifact configuration. Mara Chain
is built on a key insight: a rejected candidate is not a dead end but evidence
for the next attempt. As shown in Figure~\ref{fig:architecture}, Mara Chain is
embedded in an outer population-based loop that validates improvements and,
through Pareto-filtered Top-$N$ selection, maintains a bounded candidate pool.
All rollout logs, scores, traces, and metadata are persisted in a
filesystem-backed memory, which the LLM-based proposer can access to retrieve
and analyze the complete history of prior attempts.

\begin{figure*}[t]
\centering
\includegraphics[width=0.85\textwidth]{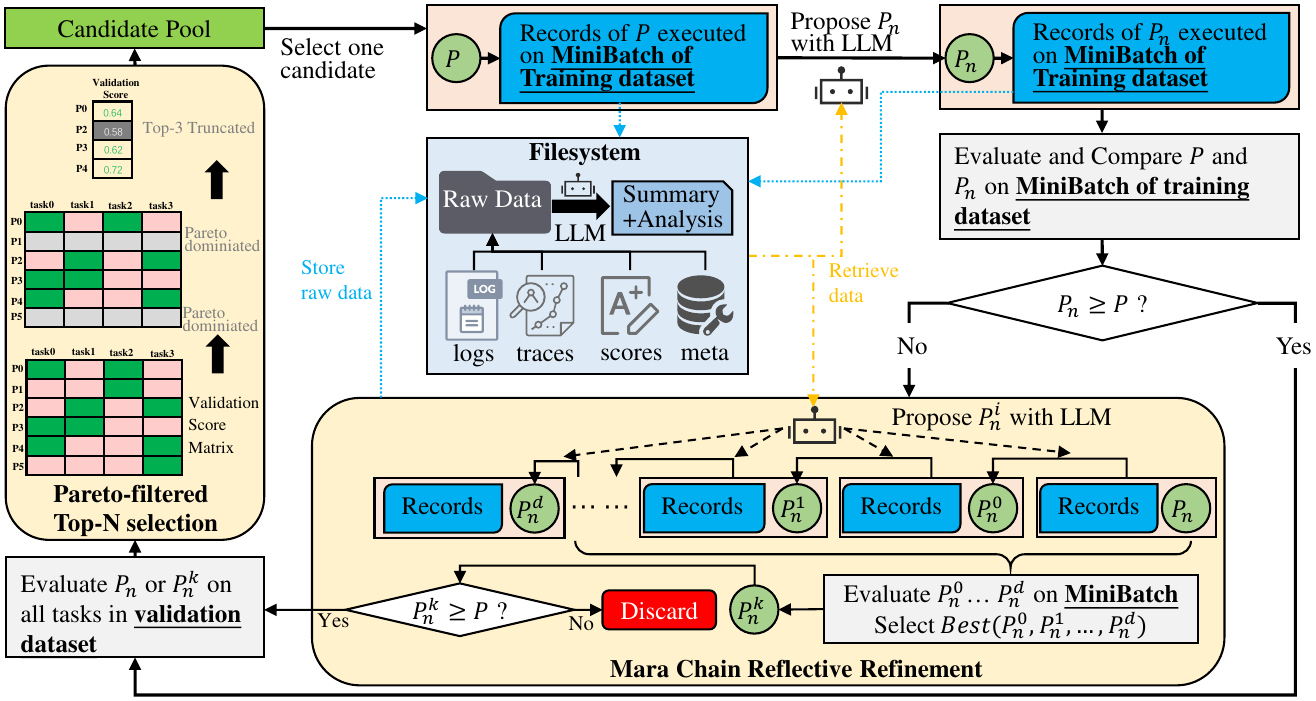}
\caption{Overview of one Mara Chain optimization slot. A parent $P$ is
sampled from the candidate pool using multi-dimensional leadership weights and
executed on a training minibatch. Rollout logs, traces, scores, and metadata are
  persisted to a filesystem-backed memory, where an analysis phase produces a structured
summary. The mutation phase uses the parent and that summary to generate $P_n$.
If $P_n$ clears the configured acceptance bar against $P$ on the same minibatch,
it is evaluated on the validation set. Otherwise, the Mara chain successively generates
$P_n^0,\ldots,P_n^d$ from the retained execution history, selects the
best-scoring chain candidate on the minibatch, and validates it only when it
clears that bar against $P$. Validated candidates undergo Pareto-filtered Top-$N$
selection: dominated
candidates are removed using validation score vectors, and a validation-score
top-$N$ truncation bounds the remaining pool. The process repeats until the
rollout budget is exhausted. Multiple slots execute this procedure in
parallel over the shared candidate pool.}
\label{fig:architecture}
\end{figure*}

\paragraph{Outer Optimization Loop.} Figure~\ref{fig:architecture} overviews the outer
propose--evaluate--select loop. The loop maintains a shared candidate pool
$\mathcal{P}$ of artifact configurations, invokes Mara Chain refinement procedure only when a
proposal fails the acceptance criterion, and runs multiple optimization slots in
parallel until the rollout budget is exhausted. Let
$B\subset\mathcal{D}_{\mathrm{tr}}$ denote a training minibatch and $V$ the
validation dataset. Each candidate $P\in\mathcal{P}$ is an artifact
configuration as defined in \S\ref{sec:problem}.

To select a parent for an available slot, Mara Chain favors candidates that 
lead on different score dimensions. For a
candidate $P$, let $z_j(P)$ denote its score on dimension $j$, and define its
leadership count as
\begin{equation}
\label{eq:leadership}
\ell(P) = \sum_j \mathbb{I}\!\left[
z_j(P) \geq \max_{Q\in\mathcal{P}} z_j(Q)-\varepsilon
\right].
\end{equation}
For multiple slots, parent candidates are sampled without replacement
with probabilities proportional to their leadership counts $\ell(P)$. This rule
favors candidates that are competitive on complementary dimensions while
retaining stochastic exploration.

The slot evaluates $P$ on $B$ to establish a minibatch baseline, then proposes
a candidate $P_n$. 
A candidate clears the acceptance bar when its
minibatch score improves sufficiently over that of its parent $P$ under
the threshold $\delta$.
If $P_n$ clears the acceptance bar against $P$, it is evaluated on all
validation tasks in $V$. Otherwise, the slot invokes Mara Chain refinement procedure. A candidate
returned by Mara Chain Reflective Refinement is forwarded to validation only if it clears the same
acceptance bar against $P$; otherwise, it is discarded. After validation, the
outer loop updates $\mathcal{P}$ using Pareto-Filtered Top-\texorpdfstring{$N$}{N} Selection. 
The loop stops when its rollout budget is exhausted.

\paragraph{Mara Chain Reflective Refinement.}
The \mc is a sequential, history-conditioned refinement procedure. It is
triggered only after the direct candidate $P_n$ does not clear the acceptance
bar against $P$ on the current minibatch. Rather than treating that rollout as discarded information,
the chain preserves the candidate, its execution records, the structured
analysis, and the artifact-change history. At chain step $i$, the proposer reads
this accumulated context and produces $P_n^i$; the rollout of $P_n^i$ then
extends the context for step $i+1$. Thus, all chain candidates form one lineage,
and every proposal is conditioned on the observed consequences of its
predecessors.

\begin{algorithm}[t]
\caption{Reflective refinement in Mara Chain.}
\label{alg:reflect}
\begin{algorithmic}[1]
\REQUIRE parent $P$ with baseline rollout $R_P$, rejected direct candidate $P_n$ with rollout $R_n$, minibatch $B$, acceptance bar $\delta$, depth $d$
\STATE $\mathcal{C} \gets [\langle P, R_P\rangle,\ \langle P_n, R_n\rangle]$ \COMMENT{Mara-chain nodes $\langle$candidate, rollout$\rangle$; root first}
\FOR{$i=0,\ldots,d$}
  \STATE $P_n^i \gets \texttt{store.create\_child}(\mathcal{C}.\mathrm{last})$ \COMMENT{inherits $\mathcal{C}.\mathrm{last}$'s artifacts}
  \STATE $\Gamma_i \gets \texttt{store.build\_context}(\mathcal{C})$ \COMMENT{score history, prior analyses, attributed changelog, run records}
  \STATE $A_i \gets \texttt{proposer.mara\_analyse}(\mathcal{C}.\mathrm{last}, \Gamma_i)$ \COMMENT{analyses persisted to \texttt{store}}
  \STATE $\texttt{proposer.mutate}(P_n^i, A_i)$ \COMMENT{edits $P_n^i$'s artifacts in place}
  \STATE $R_i \gets \texttt{rollout}(P_n^i, B)$; $s_i \gets \texttt{score}(R_i)$
  \STATE $\texttt{store.persist}(P_n^i, R_i, s_i)$; $\mathcal{C}.\texttt{add}(\langle P_n^i, R_i\rangle)$
  \STATE \textbf{if} $\textit{clears}(s_i, s_B(P), \delta)$ \textbf{then break} \COMMENT{stop once the bar is met}
\ENDFOR
\STATE $(\textit{best}, s_{\textit{best}}) \gets \argmax_{\langle C,R\rangle \in \mathcal{C}[1{:}]}\ \texttt{score}(R)$
\IF{$\textit{clears}(s_{\textit{best}}, s_B(P), \delta)$}
  \STATE \textbf{return} $\textit{best}$ \COMMENT{caller evaluates it on $V$}
\ENDIF
\STATE \textbf{return} \textsc{None}
\end{algorithmic}
\end{algorithm}

Algorithm~\ref{alg:reflect} specifies the procedure. Each chain step assembles
its context $\Gamma_i$ from the store: the per-instance score history across
all chain nodes, the structured analyses written for prior nodes, the
attributed artifact-change history along the chain, and the raw run records.
All chain nodes are scored on the same fixed
minibatch, and the chain returns its highest-scoring node only when it clears the
same configured acceptance bar as the direct candidate, stopping as soon as one
does. The store is not a
passive log: it retains raw rollout data and its LLM-generated analyses, which
make the accumulated evidence available to subsequent proposals.
The full store schema and analysis prompts are specified in App.~\ref{app:framework}.
This permits a
later candidate to extend a partial fix, repair the implementation bug of an
attempt whose direction was sound, or pursue the goal by a different route once
an earlier attempt has ruled one out. The
chain depth $d$ bounds the additional rollout cost of a failed direct proposal.



\paragraph{Pareto-Filtered Top-\texorpdfstring{$N$}{N} Selection.}
Validation produces a per-task score vector
$\mathbf{z}_V(P)\in[0,1]^{|V|}$ for each candidate. The pool update uses two
stages. First, a candidate $P$ is removed if another candidate $Q$ weakly Pareto
dominates it:
\begin{equation}
\label{eq:pareto}
Q \succeq P \quad\Longleftrightarrow\quad
z_{V,j}(Q) + \varepsilon \geq z_{V,j}(P)\quad \text{for all }j.
\end{equation}
For identical validation vectors, the more recent candidate is retained. This
per-instance criterion preserves candidates that excel on complementary subsets
of validation tasks, rather than prematurely reducing them to a single scalar
score.

Second, if the non-dominated set contains more than $N$ candidates, we retain
the $N$ candidates with the highest mean validation scores
$\frac{1}{|V|}\sum_j z_{V,j}(P)$ and remove the rest.
The first stage preserves complementary validation performance; the second
stage bounds memory and evaluation cost because high-dimensional per-task score
vectors leave many candidates mutually non-dominated, causing the candidate
pool to grow without bound under Pareto filtering alone.
Its implementation pseudocode and scheduler semantics are provided in App.~\ref{app:selim}.

\paragraph{Practical Implementation.}
We implement Mara Chain in a modular optimization framework whose pluggable
target-side, optimizer-side, and candidate-store components are instantiated per
setting with the target system and artifact schema. Component interfaces and
concrete artifact layouts are detailed in App.~\ref{app:framework}.
\section{Evaluation}
\label{sec:eval}

We evaluate Mara Chain on three benchmarks spanning distinct tunable artifact
configurations: AppWorld~\citep{trivedi2024appworld}, which optimizes skills
consisting of \texttt{SKILL.md} guidance, references, and scripts;
TerminalBench~2.1~\citep{merrill2026terminalbench}, which optimizes an
agent-harness configuration over $89$ agentic-OS tasks; and
MuSiQue~\citep{trivedi2022musique}, which optimizes a retrieval-pipeline
configuration. We conduct experiments across three frozen target LLMs:
GLM-5~\citep{zai2026glm5}, DeepSeek-V4-Pro-0813, and
Qwen3.5-397B-A17B; unless otherwise stated, GLM-5 is the default model. We use
the Pi coding agent~\citep{earendil2025pi} as the LLM-based proposer. We compare
against GEPA~\citep{agrawal2025gepa}, ACE~\citep{zhang2025ace}, and
SkillOpt-Lite~\citep{shen2026skilloptlite} for skill optimization, and
AHE~\citep{lin2026ahe} and Meta-Harness~\citep{lee2026metaharness} for
agent-harness optimization; Codex~\citep{openai2025codex} and
OpenCode~\citep{anomaly2025opencode} provide additional agent-harness references.
For MuSiQue, we compare against a hand-written default pipeline because existing
optimizers do not provide directly comparable retrieval-pipeline implementations.
App.~\ref{app:framework} specifies the implementation interfaces and artifact
schemas, while App.~\ref{app:expproto} documents the metrics, splits,
hyperparameters, and evaluation budgets.
Unless otherwise stated, we set the
Pareto-filtered Top-$N$ parameter to $N=3$ and the maximum number of Mara Chain
reflective refinement iterations to $d=5$.

\subsection{Main Results: Performance and Efficiency}
\label{sec:eval:main}


\paragraph{\textbf{AppWorld: Skills}.}
The AppWorld tunable artifacts form a \emph{skill configuration}:
natural-language \texttt{SKILL.md} guidance, references, and runnable scripts.
We evaluate on $585$ AppWorld tasks: $168$ \texttt{test\_normal} and $417$
\texttt{test\_challenge}; the training-pool and validation-set construction is
in App.~\ref{app:expproto}.

\begin{table}[t]
\centering
\small
\caption{AppWorld evaluation of optimized skill configurations on
$585$ recorded tasks ($168$ \texttt{test\_normal}; $417$
\texttt{test\_challenge}). TGC = task-grade completion,
SGC = scenario-grade completion; Normal and Challenge are the two test splits.
The green number is the improvement over the empty-skill configuration (pp).
Parentheses identify the slot configuration attaining each Mara Chain result.
Bold marks the best value per column.}
\label{tab:appworld}
\resizebox{\linewidth}{!}{%
\begin{tabular}{lllll}
\toprule
Optimized skill configuration & Normal TGC & Normal SGC & Challenge TGC & Challenge SGC \\
\midrule
Empty skill & 48.2 & 35.7 & 33.3 & 19.4 \\
SkillOpt-Lite~\citep{shen2026skilloptlite} & 70.8\up{22.6} & 51.8\up{16.1} & 68.3\up{35.0} & 46.0\up{26.6} \\
GEPA~\citep{agrawal2025gepa} & 83.3\up{35.1} & 67.9\up{32.2} & 69.3\up{36.0} & 44.6\up{25.2} \\
ACE~\citep{zhang2025ace} & 83.9\up{35.7} & 69.6\up{33.9} & 74.1\up{40.8} & 55.4\up{36.0} \\
Mara Chain & \textbf{89.9}\up{41.7} (2) & \textbf{83.9}\up{48.2} (2) & \textbf{76.7}\up{43.4} (1) & \textbf{59.0}\up{39.6} (1) \\
\bottomrule
\end{tabular}}
\end{table}

Table~\ref{tab:appworld} reports final performance across the Normal and
Challenge splits. Mara Chain attains the best result in every metric:
$89.9$/$83.9$ Normal TGC/SGC and $76.7$/$59.0$ Challenge TGC/SGC. It exceeds
the specialized skill optimizers GEPA, ACE, and SkillOpt-Lite across both task
types, showing that retained refinement improves final performance beyond
prompt-only and skill-only optimization baselines.
Figure~\ref{fig:main-results}(a) shows that Mara Chain achieves both higher
rollout efficiency and stronger final performance than the competing skill
optimizers. It reaches a validation score of $0.8$ after $3{,}280$ rollouts,
compared with $9{,}514$ for GEPA---about one third as many rollouts as the
strongest baseline; ACE and SkillOpt-Lite never reach this score. The
wall-clock view of the same trajectories shows the same gap
(Figure~\ref{fig:appworld-time}, App.~\ref{app:timecurve}): Mara Chain reaches
$0.8$ in $11.9$\,h, about $3\times$ faster than GEPA ($35.8$\,h). Mara Chain
ultimately reaches $0.87$, while GEPA plateaus at $0.805$.
As shown in Figure~\ref{fig:main-results}(a), all methods improve rapidly during the early rollouts, when many readily
diagnosed conditions remain. As these conditions are resolved, the remaining
long-horizon problems require dependent modifications and expose persistent
failure barriers. GEPA, ACE, and SkillOpt-Lite plateau in this regime, whereas
Mara Chain retains and refines evidence from rejected candidates to continue
improving. Mara Chain eventually plateaus because of the remaining
model-capability limits and its bounded depth $d$.

\paragraph{\textbf{TerminalBench~2.1: Harnesses}.}
The TerminalBench~2.1 tunable artifacts form an \emph{harness
configuration}. The harness optimizers
AHE and Meta-Harness are the baselines, with off-the-shelf CLI agents Codex and
OpenCode and the un-optimized Kira base as references. The system model is
GLM-5 across rows; the reported comparison remains end-to-end under the stated
configurations.

Figure~\ref{fig:main-results}(b) shows that Mara Chain substantially outperforms
both specialized harness optimizers and off-the-shelf coding agents on the
$89$-task TerminalBench~2.1 evaluation. The Mara-Chain-optimized harness
attains a $71.9\%$ pass rate, exceeding the next-best results of $51.7\%$ from
AHE, Codex, and the Kira base by $20.2$ percentage points, and surpassing
Meta-Harness and OpenCode ($49.4\%$) by $22.5$ points. The per-task binary
rewards underlying this comparison are provided in App.~\ref{app:tb2-cross}.

\paragraph{\textbf{MuSiQue: Retrieval Pipelines (Beyond Agentic Systems)}.}
The MuSiQue tunable artifacts form a \emph{retrieval-pipeline
configuration}: a DAG of retrieval nodes with tunable node configurations,
node scripts, and a prompt file
(App.~\ref{app:pipelinespec}). The baseline is a hand-written default pipeline,
against which we report the end-to-end optimized pipeline.
On MuSiQue (Table~\ref{tab:musique}), the optimized pipeline improves test
nDCG@10 by $+0.104$ and Recall@10 by $+0.131$ over the default, with validation
gains of $+0.134$/$+0.210$. 

\begin{table}[t]
\centering
\small
\caption{MuSiQue retrieval (4-hop). The optimized artifact is a retrieval-pipeline
configuration; the baseline is a hand-written default configuration. Metrics over top-10 retrieved
documents. Test is disjoint from the val split used for development.
w/o Mara Chain removes retained refinement
(Mara Chain depth $0$) at the same 2-slot budget. The green
number is the gain over the default baseline; Bold = best per column.}
\label{tab:musique}
\resizebox{\linewidth}{!}{%
\begin{tabular}{lllll}
\toprule
Artifact configuration & nDCG@10 (test) & Recall@10 (test) & nDCG@10 (val) & Recall@10 (val) \\
\midrule
Default (hand-written) & 0.301 & 0.329 & 0.419 & 0.395 \\
w/o Mara Chain (2-slot) & 0.371\up{0.070} & 0.421\up{0.092} & 0.505\up{0.087} & 0.560\up{0.165} \\
Mara Chain (2-slot) & \textbf{0.405}\up{0.104} & \textbf{0.460}\up{0.131} & \textbf{0.553}\up{0.134} & \textbf{0.605}\up{0.210} \\
\bottomrule
\end{tabular}}
\end{table}

These results show that Mara Chain is not limited to self-optimizing agentic
systems: whenever a system's tunable behavior can be mapped to editable
artifacts and evaluated on a target task distribution (\S\ref{sec:problem}),
the same procedure applies. MuSiQue instantiates this broader setting as a
retrieval DAG whose configurations, node scripts, and prompts are optimized
together.

\subsection{Ablation and Case Study}
\label{sec:eval:ablation}


\textbf{Effects on the final score.} Figure~\ref{fig:appworld-climb}(\emph{left}) reports the final scores of the
full setting (Mara Chain $+$ Top-$N$) versus the two ablations that disable the
Mara Chain flow or the Top-$N$ selection. Mara Chain's benefit grows with task
difficulty. On the easier Normal split its effect is barely visible: Normal
TGC is unchanged ($86.9$ vs $86.9$ w/o Mara Chain) and Normal SGC moves by
less than two points ($75.0$ vs $76.8$). On the harder Challenge split the
effect turns clearly positive: Challenge TGC $76.7$ vs $71.9$ ($+4.8$\,pp) and
Challenge SGC $59.0$ vs $51.8$ ($+7.2$\,pp). The gain is largest on the
hardest benchmark, TerminalBench~2.1, where the pass rate rises from $57.3\%$
(w/o Mara Chain) to $71.9\%$ ($+14.6$\,pp). Mara Chain's gain extends beyond
agentic benchmarks: on MuSiQue, disabling Mara Chain at the same 2-slot budget
lowers test nDCG@10 by $0.034$ and Recall@10 by $0.039$
(Table~\ref{tab:musique}). Disabling
Top-$N$ selection
instead degrades every AppWorld metric (Normal TGC/SGC $83.3$/$67.9$,
Challenge TGC/SGC $73.6$/$53.2$), so both mechanisms contribute to the full
setting's scores.

\begin{figure}[t]
\centering
\begin{minipage}[t]{0.49\textwidth}
\centering
\includegraphics[width=\linewidth]{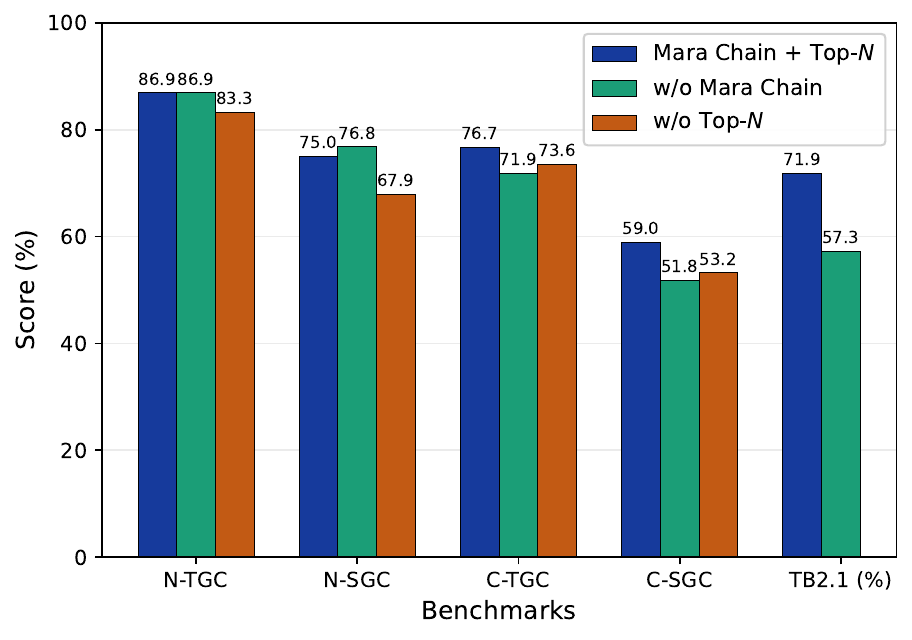}
\end{minipage}\hfill
\begin{minipage}[t]{0.49\textwidth}
\centering
\includegraphics[width=\linewidth]{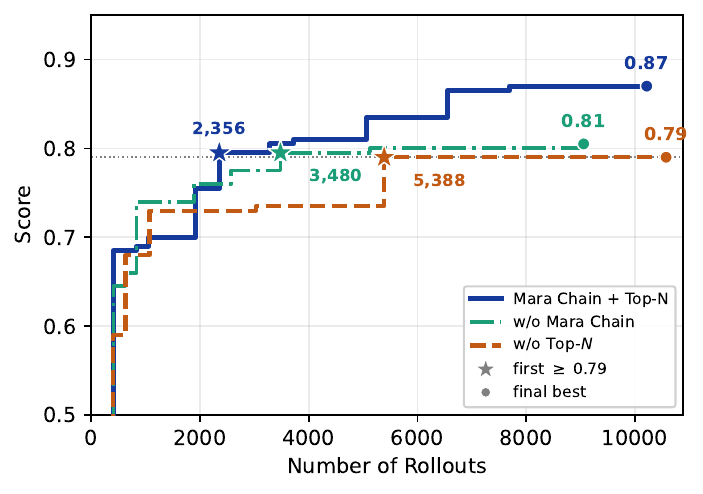}
\end{minipage}
\caption{\textbf{Left:} Ablation study of Mara Chain Reflective Refinement and Top-$N$ selection.
Mara Chain $+$ Top-$N$ enables both mechanisms; each ablation disables one of
them---the Mara Chain flow or the Top-$N$ selection---and is compared against
the both-on setting. All settings use the 1-slot configuration.
Top-$N$ selection uses $N=3$ on AppWorld and $N=1$ on
TerminalBench~2.1 (no w/o Top-$N$ result for TerminalBench~2.1).
\textbf{Right:} AppWorld validation score versus
optimization rollouts for the Mara Chain and Top-$N$ ablation.
Mara Chain with Top-$N$ reaches $0.79$ after $2{,}356$ rollouts and a final
score of $0.87$. w/o Mara Chain reaches $0.79$ after $3{,}480$ rollouts and
finishes at $0.81$, whereas w/o Top-$N$ reaches $0.79$ after $5{,}388$
rollouts. Combining Mara Chain with Top-$N$ is both the most sample-efficient
and the highest-scoring setting.}
\label{fig:appworld-climb}
\end{figure}

\textbf{Effects on sample efficiency.}
Figure~\ref{fig:appworld-climb}(\emph{right}) compares the optimization
sample efficiency of the three settings. The full setting reaches a validation
score of $0.79$ after only $2{,}356$ rollouts and continues to a final score of
$0.87$. The w/o Mara Chain ablation (Top-$N$ only) reaches $0.79$ after $3{,}480$
rollouts and finishes at $0.81$, while the w/o Top-$N$ ablation (Mara Chain only)
needs $5{,}388$ rollouts to reach $0.79$. Mara Chain and Top-$N$ are therefore
synergistic: together they are both more sample-efficient and higher-scoring
than either mechanism alone, and no single-mechanism setting attains the full
setting's final score at any comparable budget.

\begin{figure*}[!t]
\centering
\begin{minipage}[t]{0.412\textwidth}
\centering
\includegraphics[width=\linewidth,height=0.578\textheight,keepaspectratio]{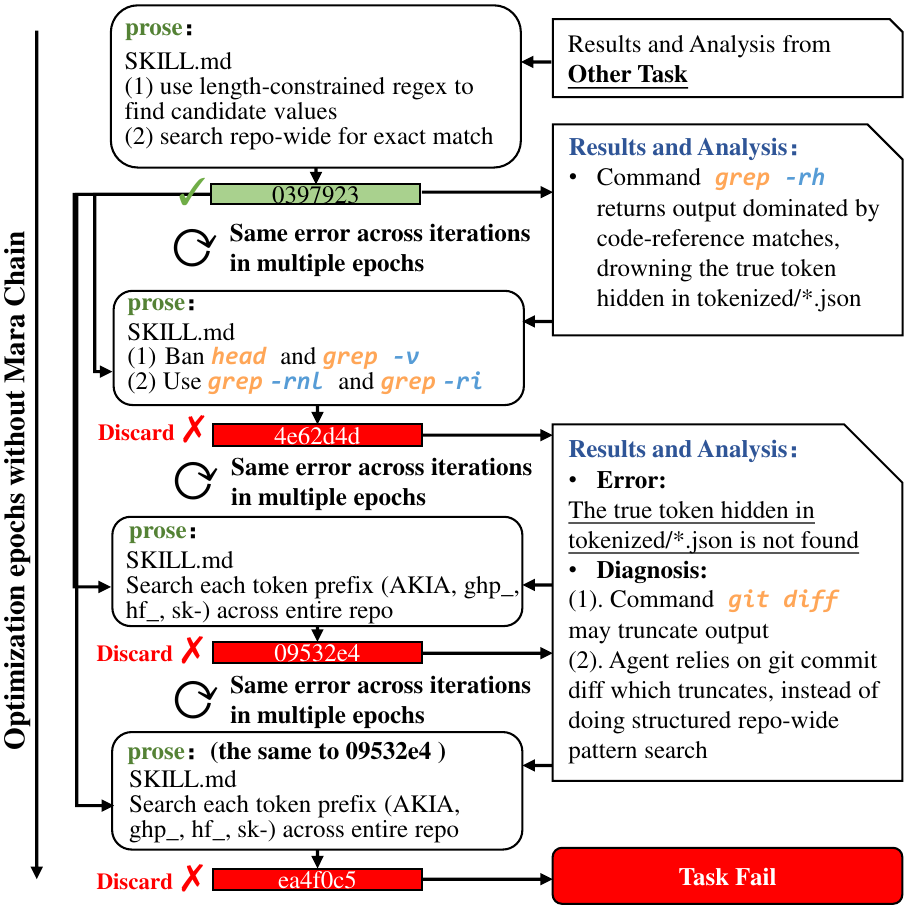}
\end{minipage}\hspace{0.06\textwidth}
\begin{minipage}[t]{0.412\textwidth}
\centering
\includegraphics[width=\linewidth,height=0.578\textheight,keepaspectratio]{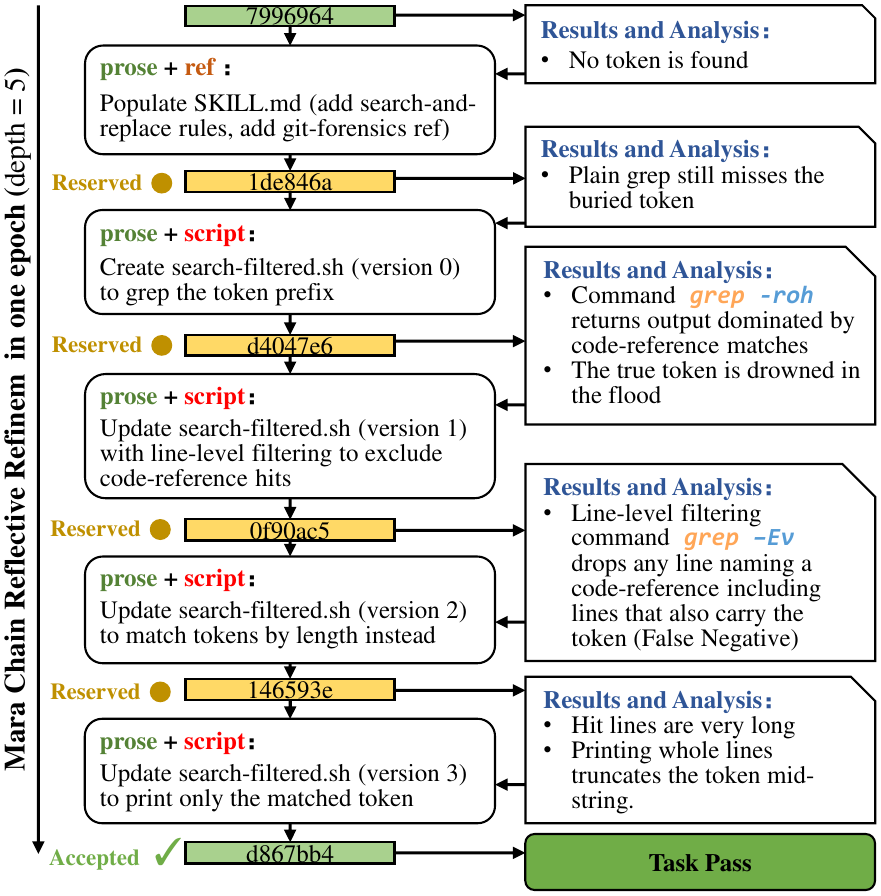}
\end{minipage}
\caption{Illustrative recorded case study of \texttt{sanitize-git-repo}.
\textbf{Left (without Mara Chain):} proposals introduce partial fixes, but each
candidate that fails to clear the acceptance bar is discarded; later iterations
therefore do not continue from these partial fixes, and the task remains failed.
\textbf{Right (with Mara Chain):} a direct candidate that does not clear the
bar is retained with its execution evidence and refined along one lineage.
Successive descendants build on partial fixes, move from prose guidance to a
script, correct the script's observed behavior, and finally pass the task.
Green, yellow, and red bars indicate accepted, reserved, and discarded
candidates.}
\label{fig:sanitize-workflow}
\end{figure*}


\textbf{Case study.}
Figure~\ref{fig:sanitize-workflow} records the \texttt{sanitize-git-repo} trace
on TerminalBench~2.1. It shows when Mara Chain Reflective Refinement matters: the task needs
several corrections, yet no single proposal clears its parent's
acceptance bar. In the displayed w/o Mara Chain setting, a candidate that
contains a partial fix but fails to clear the bar is discarded; the next propose
therefore starts from a different proposal instead of extending the partial fix,
and repeated iterations leave the task failing. Along the recorded Mara Chain
lineage, by contrast, candidates that miss the direct bar are retained with
their rollout evidence, residuals, and change history. Successive steps build on
the partial fixes, first moving from prose guidance to a script and then
correcting the script's observed behavior, until a candidate passes the task.
Notably, each refinement step along the Mara Chain lineage tends to escalate from
low-reliability prose guidance toward a higher-reliability script, so that the
final candidate's behavior is enforced by executable code rather than by
advisory text alone; in the w/o Mara Chain setting, by contrast, successive
proposals typically remain prose-level edits, whose effect is harder to enforce
reliably.
The trace thus highlights Mara Chain's advantage on long-horizon, hard
problems: instead of discarding failed candidates, it accumulates their partial
fixes along a lineage until the task passes, yielding solutions that are both
higher-performing and more reliable.
The complete recorded case material is provided in App.~\ref{app:cases}.

\subsection{Cross-Model Generalization}
\label{sec:eval:crossmodel}

We evaluate Mara Chain optimized skills and their corresponding empty-skill
configurations on AppWorld across three system models: GLM-5,
DeepSeek-V4-Pro-0813, and Qwen3.5-397B-A17B. 
Table~\ref{tab:crossmodel} shows paired evaluation differences across
all three models. On average over the four metrics, Mara Chain improves the
optimized skill over its empty-skill baseline by $+40.3$\,pp on GLM-5,
$+32.7$\,pp on DeepSeek-V4-Pro, and $+23.4$\,pp on Qwen3.5-397B-A17B. Across all
three models, Mara Chain improves every metric over the corresponding
empty-skill baseline, demonstrating that the optimization is not tied to any
single model and generalizes across different LLM models.

\begin{table}[t]
\centering
\small
\caption{Paired AppWorld evaluation of Mara Chain optimized skills versus their empty-skill configurations across three LLMs. $\Delta$ is the mean improvement across the four metrics.}
\label{tab:crossmodel}
\begin{tabular}{llccccc}
\toprule
LLM Model & Method & N-TGC & N-SGC & C-TGC & C-SGC & $\Delta$ \\
\midrule
\multirow{2}{*}{GLM-5} & Empty skill & 48.2 & 35.7 & 33.3 & 19.4 & --- \\
 & Mara Chain & 86.9 & 75.0 & 76.7 & 59.0 & +40.3 \\
\midrule
\multirow{2}{*}{DeepSeek-V4-Pro} & Empty skill & 66.7 & 39.3 & 60.9 & 30.9 & --- \\
 & Mara Chain & 93.5 & 83.9 & 81.5 & 69.8 & +32.7 \\
\midrule
\multirow{2}{*}{Qwen3.5-397B-A17B} & Empty skill & 63.7 & 48.2 & 50.1 & 30.2 & --- \\
 & Mara Chain & 82.7 & 73.2 & 73.6 & 56.1 & +23.4 \\
\bottomrule
\end{tabular}
\end{table}

\section{Related Work}
\label{sec:related}

\textbf{Prompt evolution.}
Over a \emph{prompt string}, AutoPrompt~\citep{shin2020autoprompt} searches
discrete tokens via input gradients; OPRO~\citep{yang2024opro} casts the LLM as
optimizer over a meta-prompt, APE~\citep{zhou2023ape} generates and selects
instructions, PromptBreeder~\citep{fernando2024promptbreeder} and
EvoPrompt~\citep{guo2024evoprompt} apply evolutionary mutate--select,
PromptWizard~\citep{agarwal2024promptwizard} self-evolves via critique and
synthesis, and ProTeGi~\citep{przyzant2023protegi} uses a textual gradient with
beam search. GEPA~\citep{agrawal2025gepa} applies Genetic-Pareto evolution with
reflection to the system's prompt strings. Compound-system optimizers like
TextGrad~\citep{yuksekgonul2025textgrad}, Trace~\citep{cheng2024trace},
Optimas~\citep{wu2025optimas}, Symbolic Learning~\citep{zhou2024symbolic},
DSPy~\citep{khattab2024dspy}, MIPRO~\citep{opsahlong2024mipro}, and
AFlow~\citep{zhang2024aflow} optimize pipelines or workflows, and
AlphaEvolve/FunSearch~\citep{novikov2025alphaevolve,romeraparedes2024funsearch}
and ADAS~\citep{hu2024adas} evolve code; MAP-Elites~\citep{mouret2015mapelites}
is the quality-diversity family these draw on.

\textbf{Harness evolution.}
AHE~\citep{lin2026ahe} and Meta-Harness~\citep{lee2026metaharness} search
harness code with an agentic proposer over filesystem-stored traces---AHE
wrapping analyze--mutate in a falsify-and-rollback loop, Meta-Harness leaving
\emph{how} to evolve to the agent itself. Self-Harness~\citep{zhang2026selfharness},
AutoHarness~\citep{lou2026autoharness}, and
DarwinX~\citep{zhang2026darwinx} are related variants.
ACE~\citep{zhang2025ace} evolves contexts as playbooks via
generate--reflect--curate (building on dynamic
cheatsheet~\citep{suzgun2025dynamiccheatsheet}), and SkillOpt /
SkillOpt-Lite / SkillCAT~\citep{yang2026skillopt,shen2026skilloptlite,chen2026skillcat}
optimize agent skills.

\textbf{Reflection and self-refinement.}
Reflecting on failures to drive the next attempt underlies
Reflexion~\citep{shinn2023reflexion} and
Self-Refine~\citep{madaan2023selfrefine}. Experiential agents accumulate
reusable skills or workflows---Voyager~\citep{wang2023voyager},
ExpeL~\citep{zhao2024expel}, Agent Workflow Memory~\citep{wang2024awm}---and
self-referential evolution targets the agent itself~\citep{zhang2025dgm,zhang2026hyperagents}.

\section{Limitation and Future Work}
\label{sec:limit}

\textbf{Limitations.}
Mara Chain diagnoses one batch at a time and can overfit that batch's failure
modes, which per-instance Pareto selection across batches only mitigates. Its
context construction and diagnosis logic are also a hand-designed first version
that remains fixed throughout the search.

\textbf{Future work.} One line tightens efficiency: chain depth adapts to
problem difficulty---deeper on hard barriers, earlier termination on easy
ones---with cheaper proposer models and budgeted depth reducing wall-clock and
token cost. A second widens exploration: diversity-aware selection beyond a
single test score, and branching the serial lineage into a parallel
\emph{Mara-tree}. A third makes the chain trustworthy and transferable:
per-node confidence estimates with fallback to the baseline, and effective
strategies distilled into a shared experience bank meta-learned across
artifact configurations and domains.
\section{Conclusion}
\label{sec:concl}

We introduced the \emph{\mc}, a history-conditioned refinement procedure for
the propose--evaluate--select optimization of tunable artifacts: instead of
discarding a rejected candidate, it retains the candidate's rollout traces,
residual failures, structured analysis, and artifact-change history and refines
them on the same training minibatch along a single lineage, converting evidence
from failed attempts into context for later proposals. Across skills, agent
harnesses, and retrieval pipelines (AppWorld, TerminalBench~2.1, MuSiQue), Mara
Chain outperforms specialized optimizers on both final performance and sample
efficiency, and the gains persist across GLM-5, DeepSeek-V4-Pro, and
Qwen3.5-397B-A17B. Ablations credit both mechanisms: retained refinement drives
the largest gains on the hardest tasks, and the outer loop's Pareto-filtered
Top-$N$ selection further improves final scores and sample efficiency. These
results indicate that retained, evidence-conditioned refinement is a general,
model-agnostic mechanism for AI-system auto-optimization.

%
%
\subsection*{AI use statement}

In this work, we used generative AI tools---the open-source OpenCode CLI and
an internal coding-assistant CLI, running the Kimi-K3 and GLM-5.2
models---to implement methods: the optimization framework behind our
experiments was co-developed with AI assistance through repeated
human-directed iterations, each experimental setup was built with AI
assistance, and the frontend was built with AI-generated code.
We also used these tools to organize and interpret experimental results,
including reading run logs, analyzing failure trajectories, and compiling
experimental data into the result tables, and to provide feedback on the
research methodology: a first draft of the methodology was written by AI from
phenomena the authors had observed in repeated experiments through our
run-visualization frontend, and AI was further used to inspect problem
trajectories and critique the methodology prompts. We have not used generative AI tools to generate synthetic data sets
or to formulate the mathematical claims or their proofs.

Additionally, we used generative AI tools to search for and organize related
literature and formatting requirements, to suggest the structure of this
paper, to draft and edit parts of the manuscript, and to create the figures
and tables.

The research idea and the overall architecture of the framework were
conceived by the authors; AI assistance filled in implementations and
executed code modifications under human direction, and all refactoring was
initiated and described by the authors. We have reviewed all AI-assisted
work: AI-drafted manuscript text was reviewed, revised, or rewritten by the
authors, and AI-written code was run, observed, and debugged through repeated
iterations before use. We take responsibility for the final content of this
work, including text, claims, and artifacts produced with the aid of
generative AI.

\nocite{novikov2025alphaevolve,zhang2026hyperagents,zhang2026selfharness}
\bibliography{iclr2026_conference}
\bibliographystyle{iclr2027_conference}

\appendix
\section{Framework Details}
\label{app:framework}

Figure~\ref{fig:pipeline} shows our evolutionary framework for optimizing a system by
evolving its \emph{tunable artifacts} with an agentic proposer, persisting
all state to a filesystem. We represent tunable artifacts
(\S\ref{sec:problem})---a system's editable prompts, scripts, configurations,
and runtime guards---as
a real directory of files that the optimizer mutates by adding, deleting, and
editing files in place, not a constrained parameter set or DSL---so that the
\emph{same} optimization loop can optimize systems of arbitrary form (a skill,
an agent harness, a retrieval pipeline) without re-wiring the optimizer for
each new family. The \mc (\S\ref{sec:rc}) runs on
this framework and is the focus of the paper; this appendix details the
substrate---the artifact configuration, the seven
pluggable components, the two-phase proposer, the concurrent scheduler, and the
filesystem memory. Figure~\ref{fig:pipeline} overviews the architecture and data
flow, and Algorithm~\ref{alg:mainloop} gives the full concurrent evolution loop
that invokes the chain.

\begin{figure}[ht]
\centering
\resizebox{\linewidth}{!}{
\begin{tikzpicture}[font=\footnotesize, >=Stealth, every node/.style={align=center}]
  \node[draw, rounded corners=2pt, fill=blue!6, minimum width=96mm, minimum height=11mm,
        text width=92mm] (schema) at (0,6.8)
        {\textbf{Artifact Schema} --- extensibility contract: a tunable artifact directory rendered into the proposer's prompt};

  \node[draw, rounded corners=2pt, fill=orange!10, minimum width=96mm, minimum height=13mm,
        text width=92mm] (sched) at (0,4.6)
        {\textbf{Concurrent Scheduler} (orchestrator)\\[1pt]
         \scriptsize $N$ slots over a shared pool;\;\; each slot runs \emph{evolve\_one}: run System $\to$ invoke Proposer $\to$ run System $\to$ score $\to$ [reflect] $\to$ validate $\to$ select/eliminate\\
         \scriptsize \fbox{slot$_1$}\;\; \fbox{slot$_2$}\;\; $\cdots$\;\; \fbox{slot$_N$}};

  \node[draw, rounded corners=2pt, fill=blue!4, minimum width=70mm, minimum height=15mm,
        text width=66mm] (proposer) at (0,2.5)
        {\textbf{Agentic Proposer}\\[1pt]
         \scriptsize two-phase coding agent: \emph{analyze} (reads raw runs) $\to$ \emph{mutate} (reads RunAnalysis);\\
         \scriptsize if the acceptance condition is unmet $\to$ \emph{reflect} (Mara chain, same batch)};

  \node[draw, rounded corners=2pt, fill=red!8, minimum width=44mm, minimum height=11mm,
        text width=40mm] (sys) at (-4.2,0.5)
        {\textbf{System}\\ \scriptsize runs the artifact configuration on a batch $\to$ trajectories + outputs};
  \node[draw, rounded corners=2pt, fill=red!8, minimum width=44mm, minimum height=11mm,
        text width=40mm] (ev) at (4.2,0.5)
        {\textbf{Evaluator}\\ \scriptsize scores $\in[0,1]$;\; \texttt{scoring\_criteria()} $=$ objective};

  \node[draw, rounded corners=2pt, fill=green!8, minimum width=48mm, minimum height=9mm,
        text width=44mm] (ea) at (-5.6,-1.2)
        {\textbf{EvolutionAlg plug}\\ \scriptsize select / eliminate (Pareto $+$ top-$N$)};
  \node[draw, rounded corners=2pt, fill=gray!8, minimum width=48mm, minimum height=11mm,
        text width=44mm] (mem) at (4.2,-1.2)
        {\textbf{CandidateStore}\\ \scriptsize artifact directory $\cdot$ data/ (raw runs, RunAnalysis) $\cdot$ changelog};

  \node[font=\scriptsize\itshape, gray!50!black, text width=128mm] (memnote) at (0,-2.5)
        {System runs and Evaluator scores persist to the store; analysis reads raw runs, mutation reads compressed RunAnalysis + an inherited changelog (raw runs never fed to the proposer), and both write back RunAnalysis and the mutated artifact configuration / changelog. Validation runs are isolated from the proposer; every edit is a traceable \texttt{SpecAction} that cites the failure it fixes.};

  \draw[->, thick, blue!60!black] (schema.south west) ..
        controls +(-2.6,0) and +(-2.6,0.6) ..
        node[left, font=\scriptsize\itshape, fill=white, inner sep=1pt]
        {render $\to$ proposer input\,} (proposer.north west);

  \draw[->, thick] (sched.south) --
        node[right, font=\scriptsize\itshape, fill=white, inner sep=1pt] {\,invoke propose} (proposer.north);
  \draw[->, thick] ([xshift=-3.2cm]sched.south) ..
        controls +(-0.4,-0.6) and +(0,0.6) ..
        node[left, font=\scriptsize\itshape, fill=white, inner sep=1pt] {run artifact configuration\,} (sys.north);

  \draw[->, thick] (sys.east) --
        node[above, font=\scriptsize\itshape, fill=white, inner sep=1pt] {\,trajectories} (ev.west);
  \draw[->, dashed, gray!70!black] (ev.north) ..
        controls +(0.4,1.0) and +(0,-0.4) ..
        node[right, font=\scriptsize\itshape, fill=white, inner sep=1pt] {\,score} ([xshift=3.2cm]sched.south);
  \draw[->, dashed, gray!70!black] (ev.south) --
        node[right, font=\scriptsize\itshape, fill=white, inner sep=1pt] {\,persist} (mem.north);

  \draw[->, dashed, gray!70!black] (ea.north) ..
        controls +(-0.6,2.0) and +(-0.6,-0.6) ..
        node[left, font=\scriptsize\itshape, fill=white, inner sep=1pt] {prune \& select\,}
        ([xshift=-4.4cm]sched.south);

  \draw[<->, dashed, gray!70!black] (proposer.east) .. controls +(2.4,-0.6) and +(2.4,1.4) ..
        node[right, font=\scriptsize\itshape, fill=white, inner sep=1pt]
        {\,read: raw runs / RunAnalysis\\write: artifact configuration / RunAnalysis} (mem.east);
\end{tikzpicture}}
\caption{Architecture and data flow (complements
Algorithm~\ref{alg:mainloop}). Each slot runs the candidate artifact configuration through the
system runner and scores it with the evaluator (whose scoring criteria is the
objective); the proposer is a two-phase coding agent whose analyze phase reads
raw runs and whose mutate phase reads the compressed \texttt{RunAnalysis} plus
an inherited changelog (raw runs are not fed to the mutation phase), writing the
mutated artifact configuration and \texttt{RunAnalysis} back to the candidate store. The
scheduler runs up to $N$ slots over a shared pool backed by the
evolution-algorithm plug and the candidate store. Only the artifact-configuration
instantiation and the plugins change across task families.}
\label{fig:pipeline}
\end{figure}
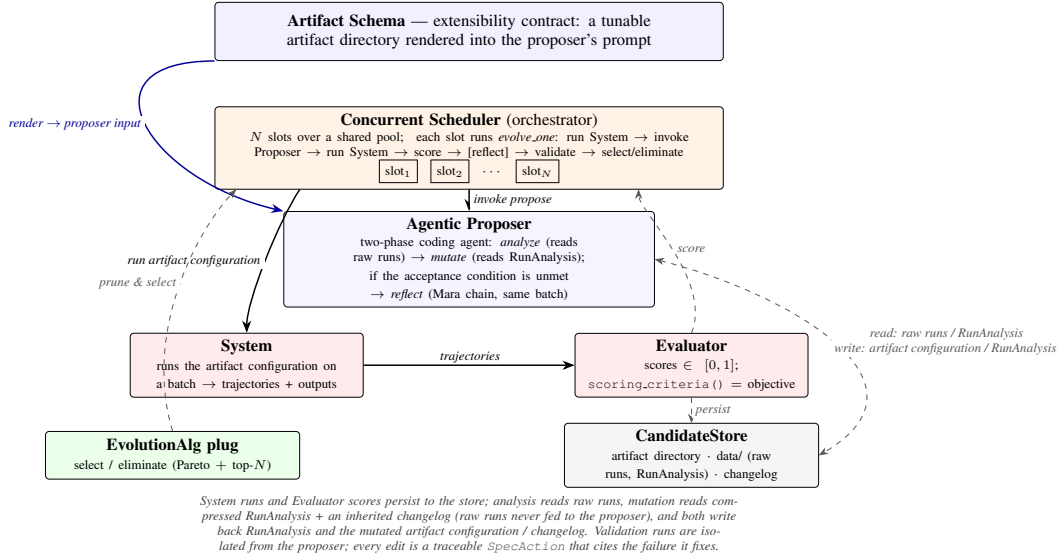

\paragraph{Seven pluggable components.}
The framework exposes seven pluggable components across three roles. The
\emph{target side} describes the system being optimized: the \emph{system
runner} runs the artifact configuration on a batch and returns trajectories and outputs; the
\emph{evaluator} scores a result and gives the reason; the \emph{data
instance} is the per-task contract. The \emph{optimizer side} drives
optimization: the \emph{optimizer} (the evolutionary loop) orchestrates
selection, proposing, evaluation, the \mc, validation, and elimination
under a budget (its default is the concurrent reflective
loop); the \emph{proposer} analyzes a failure and rewrites the artifact configuration into a
child; and the \emph{evolution algorithm} selects and eliminates candidates.
The \emph{candidate store} is the persistence layer. All seven are pluggable;
adding a system means subclassing the target side and defining an artifact
configuration.

\paragraph{Artifact schema and reliability.}
An artifact configuration $\sigma$ contains artifacts $A(\sigma)$, each with a
qualitative reliability prior $d(a)\in[0,1]$. Our concrete schemas order
representations as \texttt{prose} $<$ \texttt{script} $<$
\texttt{extension}: natural-language instructions may be ignored, scripts are
deterministic once invoked but may not be invoked, and runtime extensions are
deterministic when reached. When
rollout evidence indicates that an intended modification was bypassed, the
proposer may replace or augment it with a higher-reliability representation
permitted by the schema. The concrete directory layouts and reliability
instantiations are given in App.~\ref{app:schemas}.

\subsection{Pareto-Filtered Top-\texorpdfstring{$N$}{N} Selection}
\label{app:selim}

\paragraph{Selection and elimination pseudocode.}
The \texttt{ParetoFrontierEvolutionAlgorithm} implements multi-dimensional
leadership-weighted parent sampling and \emph{Pareto-filtered Top-$N$ selection}. 
Below is compressed pseudocode from
\texttt{evolution\_algorithm/pareto\_frontier.py}.

\textbf{select(num):}
\begin{algorithmic}[1]
\STATE $P \gets$ candidates with state=\texttt{pending}
\IF{$P = \emptyset$} \RETURN $[]$ \COMMENT{no work; scheduler waits}
\ENDIF
\IF{no candidate scored yet} \RETURN \texttt{random.sample}($P$, $k$)
\ENDIF
\FORALL{candidate $c \in P$}
  \STATE $d_c \gets$ \#dimensions where $c$ achieves $\max$ score (with $\varepsilon$ tolerance)
\ENDFOR
\STATE $k \gets \min(\text{num}, |P|)$
\STATE $S \gets$ \texttt{random.choices}($P$, weights=$\{d_c\}$, $k$) \COMMENT{weighted without replacement}
\RETURN $S$
\end{algorithmic}

\textbf{eliminate():}
\begin{algorithmic}[1]
\STATE $Q \gets$ candidates with state $\in \{\texttt{pending}, \texttt{evolving}\}$
\IF{$|Q| \leq 1$} \RETURN $[]$ \ENDIF
\STATE sort $Q$ by (\textit{generation}, \textit{created\_at}) desc
\COMMENT{Phase 1: weak Pareto dominance}
\FORALL{$a \in Q$ (alive)}
  \FORALL{$b \in Q$ (alive, $b \neq a$)}
    \IF{$a$ weakly dominates $b$} \STATE mark $b$ as dominated
    \ENDIF
  \ENDFOR
\ENDFOR
\STATE $F \gets$ alive candidates after Phase 1
\COMMENT{Phase 2: top-$N$ truncation}
\IF{$|F| > $ max\_candidate\_num}
  \STATE sort $F$ by (\textit{avg\_score}, \textit{generation}, \textit{created\_at}) desc
  \STATE retire $F[\text{max\_candidate\_num}:]$
\ENDIF
\COMMENT{only retire \texttt{pending}; \texttt{evolving} retires after finishing}
\RETURN eliminated pairs $(\text{cid}, \text{reason})$
\end{algorithmic}

\paragraph{Weak-Pareto dominance definition.}
Candidate $a$ \emph{weakly dominates} candidate $b$ (denoted $a \succeq b$) if
and only if $a$'s score is at least as good as $b$'s on every per-instance
dimension, modulo a numerical tolerance $\varepsilon$:
$$ a \succeq b \;\iff\; \forall\, i:\; a_i + \varepsilon \geq b_i,
\qquad \varepsilon = 10^{-8}. $$
This $\geq$ (not $>$) relation is used so that candidates with identical scores
are resolved by the pre-sort: deeper generation and newer creation time are
preferred, keeping the frontier biased toward more-evolved candidates. The
dominator $a$ need only be \emph{no worse} on any dimension---it need not be
strictly better. A candidate is removed only if \emph{some} other alive
candidate weakly dominates it.

\paragraph{Why top-$N$ truncation.}
Per-instance Pareto selection, adopted from GEPA, preserves complementary
winners but ceases to prune when the validation dimension $D$ grows. Under a
simple Bernoulli model, the probability that one candidate weakly dominates
another is $P(A\succeq B)=[1-p(1{-}p)]^{D}$, where $p\in(0,1)$ is a candidate's
per-instance pass probability. Because the optimized system is itself a
probabilistic model (an LLM agent), an instance outcome is stochastic rather
than certain---$p\notin\{0,1\}$---so $p(1{-}p)>0$ and the domination probability
decays exponentially with $D$, falling below $1\%$ at $D\geq20$. The Pareto
frontier then degenerates to the whole population, which bloats unchecked under
GEPA's Pareto-only selection. The top-$N$ Phase~2 truncation bounds the
population in this regime. Pairing the cap with per-instance dominance yields
\emph{Pareto-filtered Top-$N$ selection}.

\paragraph{Probability-scoring caveats under pass/fail.}
The \texttt{select} operation weights candidates by their max-score dimension
count $d_c$ (how many dimensions they lead). Under probabilistic task scoring
(each dimension is a Bernoulli trial with pass probability $p$), two effects
degrade this heuristic as the validation dimension $D$ grows:
\begin{itemize}[leftmargin=2em,itemsep=0.05em,topsep=0.05em]
  \item \emph{Frontier degeneracy.} $P(a \succeq b) = [1 - p(1{-}p)]^{D}$ falls
        below $1\%$ at $D \geq 20$, so nearly all candidates survive Phase 1 and
        the population bloats; the top-$N$ Phase 2 truncation bounds it, but the
        dimension-count weighting in \texttt{select} becomes noisy when many
        candidates share the same $d_c$ (ties are common under $p$ near $0$ or
        $1$).
  \item \emph{Sampling noise.} The weighted-without-replacement sampling
        distributes selection probability proportional to $d_c$, but $d_c$
        estimates are themselves noisy: a candidate that passed $k$ dimensions
        on a lucky batch may have $d_c$ inflated, causing it to be selected
        more often than its true quality warrants. Conversely, a candidate that
        would dominate on a \emph{different} batch may have $d_c = 0$ and be
        starved of selection opportunities.
\end{itemize}
These effects are inherent to any dominance-based EA under stochastic, sparse
rewards (pass/fail). The top-$N$ truncation mitigates population bloat; future
work could investigate adaptive dimension weighting or Thompson sampling to
address the sampling-noise issue.

\begin{table}[t]
\centering
\small
\caption{Component ablations on AppWorld, TerminalBench~2.1, and MuSiQue. The
full setting combines Mara Chain with Pareto-filtered Top-$N$ selection.
The \emph{$-$Mara Chain} setting removes retained refinement (single-shot
search, Mara-chain depth $0$) while keeping Pareto-filtered Top-$N$; the
\emph{$-$Top-$N$} setting keeps Mara Chain but relies on Pareto filtering
alone. AppWorld and TerminalBench~2.1 use the 1-slot configuration; MuSiQue
uses 2-slot. $-$Top-$N$ was not run on TerminalBench~2.1 or MuSiQue. Bold
marks the best result per metric.}
\label{tab:ablation}
\begin{tabular}{lccc}
\toprule
Metric & Mara Chain + Top-$N$ & $-$Mara Chain & $-$Top-$N$ \\
\midrule
N-TGC & \textbf{86.9} & \textbf{86.9} & 83.3 \\
N-SGC & 75.0 & \textbf{76.8} & 67.9 \\
C-TGC & \textbf{76.7} & 71.9 & 73.6 \\
C-SGC & \textbf{59.0} & 51.8 & 53.2 \\
TB2.1 (\%) & \textbf{71.9} & 57.3 & --- \\
MuSiQue nDCG@10 (test) & \textbf{0.405} & 0.371 & --- \\
MuSiQue Recall@10 (test) & \textbf{0.460} & 0.421 & --- \\
MuSiQue nDCG@10 (val) & \textbf{0.553} & 0.505 & --- \\
MuSiQue Recall@10 (val) & \textbf{0.605} & 0.560 & --- \\
\bottomrule
\end{tabular}
\end{table}

\subsection{Concurrent Scheduler}
\label{app:scheduler}

\paragraph{Concurrent scheduler.}
The scheduler runs up to $N=\texttt{num\_proposals}$ candidates concurrently
over a shared pool, so a slow slot (e.g., a hard batch in a long \mc) does not
block the others. Records carrying 1-slot and 2-slot labels differ, but do not
establish separate, additive, or chain-only effects; matched controls are
required for component attribution. Its correctness relies on the
following invariants.

\paragraph{Scheduler invariants.}
Three invariants hold throughout the concurrent loop: (1) \emph{no-select
back-off}---if \texttt{EA.select()} returns $[]$, the scheduler waits on
in-flight slots rather than deadlocking; (2) \emph{atomic state-machine
flips}---a candidate is flipped to \texttt{evolving} before its slot is
spawned, preventing double-selection (should \texttt{eliminate} ever become
\texttt{await}-able, serialization must be reintroduced); and (3)
\emph{checkpoint safety}---\texttt{evolving} candidates are reset to
\texttt{pending} on restart, and any candidate resumes from its
$(epoch,\,dataset\_index)$ cursor, with append-only JSONL preserving full
history.

\paragraph{Acceptance predicate and full concurrent loop.}
For a child score $c$, parent score $p$, and acceptance bar $\delta$, write
$\textit{clears}(c,p,\delta)$ for the benchmark- and configuration-dependent
predicate that determines admission to validation. Here $\delta$ names the
acceptance bar; the available documentation does not establish the predicate's
raw comparison semantics. The direct child is validated exactly when it clears
the bar, and the chain fires exactly when it does not; a chain winner is also
validated only when it clears the same bar.

Algorithm~\ref{alg:mainloop} is the compressed main loop each slot runs. The
\mc (Algorithm~\ref{alg:reflect}) is the only
non-mechanical step: it fires when a child misses the acceptance bar on its
batch and the depth cap $K>0$; $K{=}0$ disables the \mc and recovers
single-shot search. Parents and Mara-chain children are all scored on the
\emph{same} batch for a fair comparison, and the budget $\mathcal{B}$
(iterations / rollouts / system\_runs / tokens / elapsed) bounds the run, with
in-flight slots draining when it is exhausted.

\begin{algorithm}[h]
\caption{Main loop (compressed; full version and
concurrent-scheduler invariants in this appendix). The \mc
(Algorithm~\ref{alg:reflect}) fires when a child misses the acceptance bar.}
\label{alg:mainloop}
\begin{algorithmic}[1]
\REQUIRE determinism function $d(\cdot)$, initial configuration $\sigma_0$, candidate pool $\mathcal{P}$, budget $\mathcal{B}$, acceptance bar $\delta$, $K$ \COMMENT{Mara-chain depth cap; $K{=}0$ disables the \mc}
\STATE \textbf{concurrently} over up to $N$ slots, sharing $\mathcal{P}$ (open a slot only while $\mathcal{B}$ is unexhausted; in-flight slots drain):
\FOR{each slot that becomes free}
  \STATE $parent \gets \texttt{select}()$ \COMMENT{pick a pending candidate; if $[]$, wait on in-flight slots}
  \STATE flip $parent \to$ \texttt{evolving} \COMMENT{leave the selectable pool}
  \STATE $B \gets \texttt{derive\_batch}(parent.\textit{cursor})$
  \STATE $p \gets \texttt{evaluate}(parent,\, B)$ \COMMENT{baseline: parent on $B$}
  \IF{$p$ is full}
    \STATE \textit{skip mutation} \COMMENT{parent already perfect on $B$}
  \ELSE
    \STATE $child \gets \texttt{create\_child}(parent)$ \COMMENT{inherits artifact; cursor $\to$ next}
    \STATE $\textit{propose}(child,\, B)$ \COMMENT{analyze$\to$mutate the child}
    \STATE $c \gets \texttt{evaluate}(child,\, B)$ \COMMENT{same batch: fair comparison}
    \IF{$\textit{clears}(c,\, p,\, \delta)$}
      \STATE \textit{validate}$(child)$;\; flip $child \to$ \texttt{pending} \COMMENT{cleared the bar; joins the pool}
    \ELSE
      \IF{$K > 0$}
        \STATE $best \gets \textit{reflection\_chain}(child,\, B,\, \delta)$ \COMMENT{child missed $\Rightarrow$ reflect; Alg.~\ref{alg:reflect}}
        \IF{a winner was found}
          \STATE \textit{validate}$(best)$;\; flip $best \to$ \texttt{pending}
        \ENDIF
      \ELSE
        \STATE \textit{retire}$(child)$ \COMMENT{missed the bar, no \mc}
      \ENDIF
    \ENDIF
  \ENDIF
  \STATE advance $parent.\textit{cursor}$;\; flip $parent \to$ \texttt{pending}
  \STATE $\textit{eliminate}(\mathcal{P})$ \COMMENT{weak Pareto + top-$N$; every slot}
  \STATE \textit{persist}$(pool)$ \COMMENT{memory + checkpoint}
\ENDFOR
\end{algorithmic}
\end{algorithm}

\subsection{Proposer and Tunable-Artifact Edits}
\label{app:proposer}

\paragraph{Pi proposer extensions: anti-cheating and reliability guards.}
The proposer is an autonomous coding agent (we use the Pi coding
agent~\citep{earendil2025pi} for its unbypassable extension guards) with
filesystem and shell access. Two concerns arise from this autonomy, and
the framework addresses both through the Pi runtime's extension guard---the
\texttt{extension} guard of the schema (\S\ref{sec:problem}), which always
applies and the agent cannot bypass.

\emph{(i) Anti-cheating: validation isolation}
(\texttt{block-val-system-run.ts}). The proposer could otherwise read the
held-out validation trajectories and overfit the validation set. The extension
intercepts every \texttt{read/write/edit/bash/grep/find/ls} tool call and
blocks any whose target path matches \texttt{val\_system\_run/}:
\begin{lstlisting}
pi.on("tool_call", (event) => {
  const pathToCheck =
    event.input.file_path ?? event.input.command ?? /* ... */;
  if (/val_system_run/.test(pathToCheck))
    return { block: true,
             reason: "Access to val_system_run is not allowed ..." };
});
\end{lstlisting}
The proposer therefore cannot read validation-set trajectories, so the
validation score reflects generalization, not leakage.

\emph{(ii) Reliability: enforced changelog discipline}
(\texttt{enforce-changelog-cli.ts}). Because LLM instruction-following is
probabilistic, the proposer frequently ignores the instruction to use the
sanctioned CLI and writes \texttt{changelog.jsonl} itself, often producing
malformed JSONL (e.g., several JSON objects concatenated on one line) that
breaks downstream parsing. The extension blocks direct writes to
\texttt{changelog.jsonl}---via \texttt{write/edit} or shell
redirection/\texttt{tee}---and forces the proposer to record every edit through
the \texttt{append-changelog} CLI, which emits one well-formed JSONL object per
edit. This minimizes the probability of a malformed lineage record rather than
relying on the model to follow the format.

These guards constrain the proposer's access to the validation signal and its
lineage bookkeeping. They are a concrete instance of the ``hard runtime guard
the model cannot bypass'' instantiated by the schema's \texttt{extension} guard.

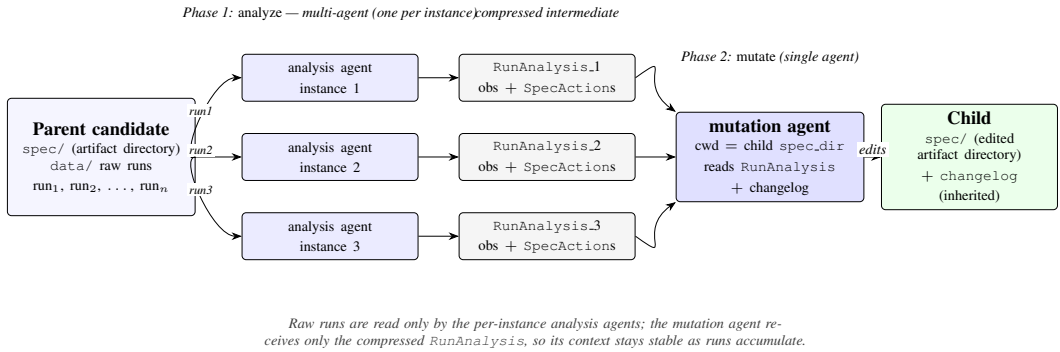
\begin{figure}[ht]
\centering
\resizebox{\linewidth}{!}{
\begin{tikzpicture}[font=\footnotesize, >=Stealth, every node/.style={align=center}]
  \node[draw, rounded corners=2pt, fill=blue!4, minimum width=32mm, minimum height=20mm,
        text width=28mm] (parent) at (-5.6,0)
        {\textbf{Parent candidate}\\[1pt]
         \scriptsize \texttt{spec/} (artifact directory)\\
         \texttt{data/} raw runs\\
         run$_1$, run$_2$, $\dots$, run$_n$};

  \foreach \i/\y in {1/1.35, 2/0, 3/-1.35}{
    \node[draw, rounded corners=2pt, fill=blue!8, minimum width=30mm, minimum height=8mm,
          text width=26mm] (an\i) at (-1.7,\y)
          {\scriptsize analysis agent\\ instance \i};
    \node[draw, rounded corners=2pt, fill=gray!8, minimum width=30mm, minimum height=8mm,
          text width=26mm] (ra\i) at (2.0,\y)
          {\scriptsize \texttt{RunAnalysis}\_\i\\ obs $+$ \texttt{SpecAction}s};
    \draw[->] (parent.east) .. controls +(-0.3,0) and +(-0.5,0) ..
          node[above, font=\tiny\itshape, fill=white, inner sep=1pt]{run\i} (an\i.west);
    \draw[->] (an\i.east) -- (ra\i.west);
  }
  \node[font=\scriptsize\itshape] at (-1.7,2.45) {Phase 1: \emph{analyze} --- multi-agent (one per instance)};
  \node[font=\scriptsize\itshape] at (2.0,2.45) {compressed intermediate};

  \node[draw, rounded corners=2pt, fill=blue!12, minimum width=32mm, minimum height=12mm,
        text width=28mm] (mut) at (5.8,0)
        {\textbf{mutation agent}\\[1pt]
         \scriptsize cwd $=$ child \texttt{spec\_dir}\\
         reads \texttt{RunAnalysis} $+$ changelog};
  \draw[->] (ra1.east) .. controls +(0.7,0.5) and +(-0.7,0.5) .. (mut.north west);
  \draw[->] (ra2.east) -- (mut.west);
  \draw[->] (ra3.east) .. controls +(0.7,-0.5) and +(-0.7,-0.5) .. (mut.south west);
  \node[font=\scriptsize\itshape] at (5.8,1.7) {Phase 2: \emph{mutate} (single agent)};

  \node[draw, rounded corners=2pt, fill=green!8, minimum width=30mm, minimum height=12mm,
        text width=26mm] (child) at (9.2,0)
        {\textbf{Child}\\[1pt]
         \scriptsize \texttt{spec/} (edited artifact directory)\\
         $+$ \texttt{changelog} (inherited)};
  \draw[->] (mut.east) -- node[above, font=\scriptsize\itshape, fill=white, inner sep=1pt]{edits} (child.west);

  \node[font=\scriptsize\itshape, gray!50!black, text width=130mm] at (1.8,-3.0)
        {Raw runs are read only by the per-instance analysis agents; the mutation agent
         receives only the compressed \texttt{RunAnalysis}, so its context stays stable as
         runs accumulate.};
\end{tikzpicture}}
\caption{Two-phase propose. \emph{Analyze} is multi-agent: one small analysis
agent per data instance reads that instance's raw run (in the parent's data
directory) and emits a structured \texttt{RunAnalysis}. \emph{Mutate} is a
single agent that reads only the compressed \texttt{RunAnalysis} set (plus an
inherited changelog) and edits the child's tunable artifacts---raw runs are
never fed to it,
so its context stays stable as runs accumulate.}
\label{fig:proposer}
\end{figure}

\paragraph{ArtifactAction and ChangeLog.}
Every mutation is a structured \texttt{ArtifactAction} that must cite at least one
traced failure---an AHE-style falsifiable contract: an action with non-empty
\texttt{resolves} referencing no observed failure is rejected. Each accepted
action is appended to the candidate's \texttt{changelog.jsonl} as a
\texttt{ChangeLogEntry}, and a child inherits its parent's changelog, so the
attributed edit history accumulates down the lineage. Their schemas:
\begin{lstlisting}
ArtifactAction {             // one tunable-artifact edit
  file: str                  // concrete path relative to the artifact root
  operation: "add" | "delete" | "modify"
  artifact_issue: str        // the defect in this file (what is wrong / missing)
  change: str                // exact change; BEFORE/AFTER format for "modify"
  resolves: list[int]        // indices into RunAnalysis.trajectory_analysis;
                             // non-empty -- must cite >=1 traced failure
}

ChangeLogEntry {             // one line in changelog.jsonl (inherited by children)
  timestamp: datetime
  type: "feat" | "fix" | "refactor" | "perf" | "docs" | "style" | "chore"
  subject: str               // imperative summary, <= 72 chars
  body: str                  // observed failure pattern + change + expected impact
  diff: str                  // unified (git) diff of the actual changes
  files_modified: list[str]
  author: str                // proposer identity (e.g. pi, claude_code)
}
\end{lstlisting}
\noindent The candidate state machine is
$\texttt{pending} \leftrightarrow \texttt{evolving} \to \texttt{unavailable}$:
\texttt{pending} candidates are selectable by the evolution algorithm,
\texttt{evolving} means a slot owns the candidate, and \texttt{unavailable} is
used by in-flight \mc children.

\paragraph{Proposer input and output.}
The mutation phase's operating prompt is assembled at render time from four
sources, none hardcoded in proposer code: (i) the
\emph{rendered schema}---the component tree plus each component's role and
determinism; (ii) the \emph{objective}, taken from the evaluator's
\texttt{scoring\_criteria()}; (iii) the \emph{diagnosis contract}, the
\texttt{RunAnalysis}/\texttt{ArtifactAction} schemas the analysis phase must
produce; and (iv) the \emph{compressed inputs from the analyze phase}---the
per-instance \texttt{RunAnalysis} set and the inherited changelog, not raw
runs. The mutation agent runs with the child's \texttt{artifact\_dir} as its working
directory.
The determinism ordering travels with the schema: when a fix at the
current component is bypassed, the rendered prompt instructs the proposer to
move it to a higher-determinism asset rather than restating it, so the
instruction is not hardcoded in the loop. This is the weld between the schema and
the \mc---the chain drives edits down the determinism gradient, and the
gradient itself is injected from the schema.

What the proposer ultimately delivers is a new child candidate: the edited
\texttt{artifact/} directory (the genome, mutated by applying the committed
\texttt{ArtifactAction}s) plus an appended \texttt{ChangeLogEntry} recording each
edit and the failure it resolves. The candidate is then scored by the evaluator
and, if accepted, enters the population with this attributed lineage inherited
by its own children. Full prompt templates and the rendering pipeline are in the
released code.

\subsection{Filesystem as Full-Fidelity Memory}
\label{app:memory}

\paragraph{Filesystem as full-fidelity memory.}
The filesystem is the substrate that makes \emph{optimize anything} tractable.
An evolution run produces far more data---full execution trajectories,
per-instance scores, the analyzer's structured \texttt{RunAnalysis}, the
proposer's attributed edits, and the lineage linking each child to its
parent---than could be held in memory or stuffed into an LLM context window.
Keeping all of it on a filesystem lets the analysis phase access the raw records
and the mutation phase access its structured outputs, including layouts that are
hard to anticipate in advance.
Persistence sits behind the candidate-store interface: the default is
a local filesystem, but it can be reimplemented to sync to a remote or cloud
backend, so the loop is not coupled to a single machine or storage medium.
This follows the filesystem-as-memory idea of AHE and
Meta-Harness~\citep{lin2026ahe,lee2026metaharness}; the on-disk layout is shown in
Fig.~\ref{fig:disklayout} and detailed below.

We keep this data on a filesystem rather than in an in-memory or
prompt-compressed representation for two reasons, each grounded in what we
observed during evolution runs:
\begin{itemize}[leftmargin=2em,itemsep=0.15em,topsep=0.15em]
\item \emph{No information loss, and drill-down on demand.} Keeping run and
        analysis data on disk means the analysis phase can drill down from a
        compressed \texttt{RunAnalysis} into the raw trajectory when a
        diagnosis is uncertain; the mutation phase receives the structured
        analysis rather than raw trajectories.
  \item \emph{Reproducibility and resumability.} Append-only logging makes
        every run reproducible and resumable from any checkpoint, which
        matters for evolution runs that span days and must survive crashes.
\end{itemize}

\paragraph{Licensed compression.}
The mutation phase reads a \emph{compressed} \texttt{RunAnalysis}, not the raw
runs, so its context stays stable as runs accumulate. This compression is
\emph{licensed}: each \texttt{RunAnalysis} is produced by a small per-instance
analysis agent whose job is small enough that distillation is deliberate, and
the raw runs remain on disk for drill-down---so it is not the trace-destroying
compression Meta-Harness~\citep{lee2026metaharness} warns against, but a
structured evidence layer on top of full-fidelity logs.

\paragraph{On-disk layout of a run's workspace.}
A run is a timestamped workspace; Figure~\ref{fig:disklayout} shows its
directory tree. Under \texttt{candidates/\{candidate\_id\}/}, \texttt{artifact/} is
the genome (inherited and mutated) and \texttt{data/} is the phenotypic memory,
holding \texttt{meta.json} (lineage, lifecycle state
$\texttt{pending}/\texttt{evolving}/\texttt{unavailable}$, and the per-candidate
progress cursor), an append-only \texttt{changelog.jsonl} (the attributed edit
history), \texttt{summary.json} (per-instance scores), the raw system and
validation runs, and the proposer's analysis and mutation artifacts. The
\texttt{val\_system\_run/} directory is hard-isolated from the proposer,
enforcing the validation isolation above. The sibling \texttt{logs/} directory
holds run-level artifacts that span the whole optimization:
\texttt{statistics.json} (best score, score history, candidate counts, usage),
an append-only \texttt{iteration\_records.jsonl} (one record per iteration),
and \texttt{parameters.jsonl} (the run's parameters).

\begin{figure}[H]
\centering
\begin{minipage}{0.99\linewidth}
\scriptsize\ttfamily\begin{verbatim}
workspace_dir/
+-- candidates/
|   \-- {candidate_id}/
|       +-- artifact/              # genome: inherited & mutated
|       \-- data/                  # phenotypic memory
|           +-- meta.json          # lineage, state, progress cursor
|           +-- changelog.jsonl    # append-only attributed edit history
|           +-- summary.json       # per-instance scores
|           +-- system_run/{data_id}/      # system run trajectories
|           +-- val_system_run/{data_id}/  # validation runs (proposer-forbidden)
|           \-- proposer_run/               # proposal agents' artifacts
|               +-- analysis/
|               |   +-- trajectory/{child_id}.json  # analyzer trajectory
|               |   \-- result/{data_id}.json       # RunAnalysis (diagnoses)
|               \-- mutation/{child_id}.json        # proposer trajectory & edits
\-- logs/                          # run-level, global across candidates
    +-- statistics.json            # optimization stats (best score, history, counts)
    +-- iteration_records.jsonl    # append-only, one record per iteration
    \-- parameters.jsonl           # run parameters
\end{verbatim}
\end{minipage}
\caption{On-disk layout of a run's workspace in our store.
\texttt{candidates/\{candidate\_id\}/} holds each candidate's \texttt{artifact/}
genome and \texttt{data/} phenotypic memory (run data, the analyzer's diagnoses
in \texttt{analysis/result/}, the proposer's edits in \texttt{mutation/}, and
the lineage); \texttt{logs/} holds the run-level statistics, iteration records,
and parameters.}
\label{fig:disklayout}
\end{figure}

\section{Artifact Schemas}
\label{app:schemas}

Each of the three benchmarks instantiates a different artifact configuration---a
directory of files with an artifact-reliability structure (\S\ref{sec:problem}).
Below we give the on-disk layout and the \texttt{TunableArtifactSchema} code
for each, and
Table~\ref{tab:schema-layers} makes the artifact-reliability instantiation
concrete: \emph{prose} artifacts have $d<\theta$ and may be bypassed,
\emph{script} artifacts execute deterministically once invoked (but invocation
may be skipped), and \emph{extension} artifacts execute deterministically when
reached. Mara Chain
inherits the artifact configuration along one lineage and can replace a
repeatedly bypassed prose artifact with a script or extension when supported by
the rollout evidence.

\begin{table}[ht]
\centering\small
\caption{Asset/determinism instantiation of each artifact configuration. $d(a)$ is a
qualitative prior, not a measured probability. Scripts are deterministic after
invocation, and ``on-demand'' assets are loaded only when the model invokes
them.}
\label{tab:schema-layers}
\setlength{\tabcolsep}{4pt}
\begin{tabularx}{\linewidth}{@{}l l l l Y@{}}
\toprule
Artifact configuration & Asset & Type & $d(a)$ & Role at runtime \\
\midrule
AppWorld & \texttt{skill/SKILL.md} & prose & $<\theta$ & strategy instructions (agent may ignore) \\
 & \texttt{skill/references/} & prose & $<\theta$ & API docs/examples, on-demand \\
 & \texttt{skill/scripts/} & script & $\ge\theta$ & verification helpers (deterministic when invoked) \\
\addlinespace
TerminalBench & \texttt{agent/kira\_agent.py} & extension & $\approx 1$ & harness code: loop/tools/middleware, cannot bypass \\
 & \texttt{agent/prompt-templates/} & prose & $<\theta$ & system prompt template (NL) \\
 & \texttt{skill/SKILL.md} & prose & $<\theta$ & task-solving strategy \\
 & \texttt{skill/scripts/} & script & $\ge\theta$ & verification helpers \\
 & \texttt{skill/references/} & prose & $<\theta$ & detailed docs, on-demand \\
\addlinespace
MuSiQue & \texttt{pipeline.json} & config & $\approx 1$ & DAG: node order + params, executed by runner \\
 & \texttt{nodes/*.py} & script & $\ge\theta$ & node implementations, executed, cannot bypass \\
 & \texttt{prompt/rewrite.md} & prose & $<\theta$ & rewrite prompt, read but not enforced \\
\bottomrule
\end{tabularx}
\end{table}

\paragraph{AppWorld (skill artifact configuration).}
{\scriptsize\begin{verbatim}
appworld-skill/
+-- skill/                      # NL instructions (agent may ignore)
|   +-- SKILL.md                # AppWorld strategy instructions
|   +-- references/             # API docs, examples (loaded on demand)
|   +-- scripts/                # verification helpers
|   +-- (additional files)
\end{verbatim}}
\texttt{APPWORLD\_SKILL\_TUNABLE\_ARTIFACT\_SCHEMA} in
\texttt{appworld\_tunable\_artifact\_def.py}:
{\scriptsize\begin{verbatim}
APPWORLD_SKILL_TUNABLE_ARTIFACT_SCHEMA: TunableArtifactSchema = FolderSchema(
    name="appworld-skill",
    files=[
        FolderSchema(name="skill", files=[
            FileSchema(name="SKILL.md",
                       description=_APPWORLD_SKILL_MD_DESCRIPTION),
            FolderSchema(name="scripts",
                         description=_SCRIPTS_DIR_DESCRIPTION),
            FolderSchema(name="references",
                         description=_APPWORLD_REFERENCES_DIR_DESCRIPTION),
            FileSchema(name="...",
                       description=_ADDITIONAL_FILES_DESCRIPTION),
        ]),
    ],
)
\end{verbatim}}
The skill contains prose, references, and optional scripts. A script is
deterministic after invocation but does not guarantee invocation; the observed
case trace records one progression to a script, not a necessary repair rule.

\paragraph{TerminalBench~2.1 (agent-harness artifact configuration).}
{\scriptsize\begin{verbatim}
terminalbench-kira/
+-- agent/                      # harness code (deterministic, cannot bypass)
|   +-- kira_agent.py           # AgentHarness class: agent loop, tools,
|   |                           # reminder/middleware, budget enforcement
|   +-- prompt-templates/
|   |   \-- terminus-kira.txt   # system prompt template (str.format)
|   +-- (additional files)
+-- skill/                      # NL instructions (agent may ignore)
    +-- SKILL.md                # task-solving strategy
    +-- scripts/                # verification helpers
    +-- references/             # detailed docs (loaded on demand)
    +-- (additional files)
\end{verbatim}}
\texttt{TERMINALBENCH\_KIRA\_TUNABLE\_ARTIFACT\_SCHEMA} in
\texttt{terminalbench\_kira\_tunable\_artifact\_def.py}:
{\scriptsize\begin{verbatim}
TERMINALBENCH_KIRA_TUNABLE_ARTIFACT_SCHEMA: TunableArtifactSchema = FolderSchema(
    name="terminalbench-kira",
    description=_KIRA_ARTIFACT_DIR_DESCRIPTION,
    files=[
        FolderSchema(
            name="agent",
            description=_AGENT_DIR_DESCRIPTION,
            files=[
                FileSchema(name="kira_agent.py",
                           description=_AGENT_MAIN_DESCRIPTION),
                FolderSchema(name="prompt-templates",
                             description=_PROMPT_TEMPLATES_DIR_DESCRIPTION,
                             files=[
                    FileSchema(name="terminus-kira.txt",
                               description=_PROMPT_TEMPLATE_DESCRIPTION),
                ]),
                FileSchema(name="...",
                           description=_ADDITIONAL_FILES_DESCRIPTION),
            ],
        ),
        FolderSchema(
            name="skill",
            description=_KIRA_SKILL_DIR_DESCRIPTION,
            files=[
                FileSchema(name="SKILL.md",
                           description=_KIRA_SKILL_MD_DESCRIPTION),
                FolderSchema(name="scripts",
                             description=_SCRIPTS_DIR_DESCRIPTION),
                FolderSchema(name="references",
                             description=_REFERENCES_DIR_DESCRIPTION),
                FileSchema(name="...",
                           description=_ADDITIONAL_FILES_DESCRIPTION),
            ],
        ),
    ],
)
\end{verbatim}}
\texttt{agent/} is the unbypassable extension guard (harness code always
runs); \texttt{skill/} is bypassable prose. These types permit edits such as
\texttt{kira\_agent.py} code at extension and \texttt{SKILL.md} rules at prose.

\paragraph{MuSiQue (retrieval-pipeline artifact configuration).}
{\scriptsize\begin{verbatim}
rag-pipeline/
+-- pipeline.json               # DAG config: node order + params (k, top_m, ...)
+-- nodes/                      # node implementations (deterministic Python)
|   +-- recall_bm25.py          #   BM25 retriever
|   +-- recall_dense.py         #   dense retriever (embedding)
|   +-- rerank_cross.py         #   cross-encoder rerank
|   +-- graph_rescore.py        #   graph-based rescore
|   +-- truncate.py             #   top-k truncation
|   +-- query_rewrite.py        #   LLM query rewrite (was skeleton)
|   +-- fuse_rrf.py             #   RRF fusion (added by optimizer)
|   +-- graph_expand.py         #   entity-overlap expansion (added)
\-- prompt/
    \-- rewrite.md              # query rewrite prompt (NL, node reads on demand)
\end{verbatim}}
\texttt{RAG\_PIPELINE\_TUNABLE\_ARTIFACT\_SCHEMA} in
\texttt{rag\_pipeline\_tunable\_artifact\_def.py}:
{\scriptsize\begin{verbatim}
RAG_PIPELINE_TUNABLE_ARTIFACT_SCHEMA: TunableArtifactSchema = FolderSchema(
    name="rag-pipeline",
    description=_ARTIFACT_DIR,
    files=[
        FileSchema(name="pipeline.json",
                   description=_PIPELINE_JSON),
        FolderSchema(name="nodes",
                     description=_NODES, files=[]),
        FolderSchema(name="prompt",
                     description=_PROMPT_DIR, files=[]),
        FileSchema(name="...",
                   description=_ADDITIONAL_FILES_DESCRIPTION),
    ],
)
\end{verbatim}}
\texttt{pipeline.json} and \texttt{nodes/} are deterministic (the runner
executes them, so $d\ge\theta$); the \texttt{prompt/} directory contains NL
prompts the system reads but does not enforce. The before/after evolved DAG is
in App.~\ref{app:pipelinespec}.

\section{Formal Conditions and Proofs}
\label{app:reflect:formal}

These statements are generic to any proposal process satisfying the following
assumptions; they do not prove that Mara Chain improves any probability. For a
task $\tau$, let $G_\tau=(C,\prec)$ be a finite directed acyclic graph
of conditions, where predecessors of $c$ must be satisfied before $c$ is
observable. Let $\mathrm{Sat}(\sigma)\subseteq C$ be the conditions satisfied
by configuration $\sigma$, and let
\[
F(\sigma)=\{c\in C\setminus\mathrm{Sat}(\sigma):
\mathrm{pred}(c)\subseteq\mathrm{Sat}(\sigma)\}.
\]
We assume \emph{residual completeness}: the evaluator reports every currently
unsatisfied observable condition, so $\delta(\sigma,\tau)=F(\sigma)$. An
accepted update is required to be non-regressing:
$\mathrm{Sat}(\sigma)\subseteq\mathrm{Sat}(\sigma')$. These are assumptions
about the abstraction, not properties established by the experiments.

\begin{theorem}[Persistent residual under unsupported accepted updates]
\label{thm:wall}
Let $c^\star\in\delta(\sigma,\tau)$. Under residual completeness and
non-regressing accepted-update semantics, if every accepted proposal distribution
assigns zero probability to updates that satisfy $c^\star$, then
$c^\star\in\delta(\sigma_t,\tau)$ after every accepted update $t$.
\end{theorem}

\noindent\emph{Proof.} An accepted update cannot remove $c^\star$ by
assumption. Non-regression preserves the predecessors that make it observable,
and residual completeness reports it. Induction over accepted updates proves the
claim. \hfill $\square$

\paragraph{Realizability and conditional support.}
For every condition $c$ that becomes exposed under a configuration satisfying
its predecessors, \emph{realizability} means that at least one admissible,
accepted update exists that satisfies $c$ while preserving all currently
satisfied conditions. Let $\mathcal{H}_{c,j}$ be the sigma field generated by
the complete history before retained attempt $j$ at exposed condition $c$.

\begin{theorem}[Conditional repeated-attempt bound]
\label{thm:rc}
Suppose residual completeness, realizability, and non-regressing accepted
updates hold. Further assume, rather than infer, that for every exposed $c$ and
attempt $j$, $\Pr(\text{attempt $j$ proposes and accepts an update satisfying
$c$}\mid\mathcal{H}_{c,j})\ge p>0$. If $G_\tau$ has at most
$n$ conditions and
\[
m\geq\left\lceil\frac{\log(n/\eta)}{-\log(1-p)}\right\rceil,
\]
then allocating $m$ attempts to each exposed condition satisfies all conditions
with probability at least $1-\eta$, using at most $nm$ attempts.
\end{theorem}

\noindent\emph{Proof.} For an exposed condition, the conditional probability
that all $m$ attempts fail is at most $(1-p)^m\leq\eta/n$. Non-regression and
residual completeness expose successors in topological order. A union bound over
at most $n$ conditions gives the result. \hfill $\square$

\section{\MCc: Full Detail}
\label{app:reflect}

\subsection{Chain context}
\label{app:reflect:context}

For a direct proposal, the documented analysis phase reads the candidate's raw
run records, while the direct mutation phase receives only the resulting
compressed \texttt{RunAnalysis} and inherited changelog. For a Mara-chain
proposal, the documentation specifies the following lineage material for the
analysis phase, all pinned to the chain and the single failing instance:
\begin{itemize}[leftmargin=2em,itemsep=0.2em,topsep=0.2em]
  \item \textbf{Lineage path} ($v_0\to\cdots\to v_{k-1}$): the ordered sequence
        of prior failed children the node is continuing.
  \item \textbf{Per-instance score table with two delta columns}: each prior
        node's score plus $\Delta_{\text{root}\to\text{cid}}$ (cumulative gain)
        and $\Delta_{\text{prev}\to\text{cid}}$ (gain over the preceding node).
  \item \textbf{Attributed changelog}: the inherited edit history with every
        entry tagged by the candidate that authored it.
  \item \textbf{Prior nodes' \texttt{RunAnalysis} diagnoses}: inherited
        diagnosis content used to condition the next refinement step.
  \item \textbf{Prior nodes' run files}: the latest trajectory of each prior
        node on this instance, for cross-version comparison.
\end{itemize}
The chain mutation phase receives the resulting analyses and changelog rather
than raw run files. The documentation does not specify how any pre-loading,
direct file access, and store-mediated access interact within the analysis
phase; this scope is therefore unspecified. The node also receives the
inherited residual $\delta_{k-1}$.

\subsection{Five-method diagnosis and the category table}
\label{app:reflect:diag}

The Mara-chain analyze phase forces five diagnostic methods, each a concrete
operation on the cross-attempt context:
\begin{enumerate}[leftmargin=2em,itemsep=0.2em,topsep=0.2em]
  \item \textbf{Delta-guided evidence tracing.} Attack the inherited residual.
  \item \textbf{Success extraction.} Lift what a passing roll did right.
  \item \textbf{Oscillation detection.} A symptom that flips across nodes
        signals a regression to undo.
  \item \textbf{Inconsistency-pattern discovery.} Variance under the
        \emph{same} artifact configuration is nondeterminism, not an artifact bug.
  \item \textbf{Run-evidence corroboration.} Every diagnosis must cite a span
        in the raw run files.
\end{enumerate}

Each prior-attempt diagnosis is classified into a structured category table:
\begin{itemize}[leftmargin=2.4em,itemsep=0.1em,topsep=0.1em]
  \item[\textbf{A}] \emph{never-triggered}---the new content was never reached;
        fix the trigger or move it where it always fires.
  \item[\textbf{B}] \emph{content-wrong}---fired but wrong; REMOVE or REPLACE.
  \item[\textbf{C}] \emph{vague}---correct in spirit but unactionable; replace
        with a concrete WHEN/DO/VERIFY procedure.
  \item[\textbf{D}] \emph{conflict}---overridden by other content; resolve and
        cite the conflicting candidate.
  \item[\textbf{E}] \emph{root-cause-irrelevant}---the failure has a different
        cause.
  \item[\textbf{F}] \textit{direction-exhausted}---tried repeatedly with no
        gain; switch component or mechanism.
  \item[\textbf{G}] \emph{inconsistent}---same artifact configuration, different outcome;
        harden the winning path.
  \item[\textbf{H}] \emph{lost-win}---an effective change was removed or
        overridden; RESTORE and guard it.
  \item[\textbf{I}] \emph{mechanism-too-weak}---the artifact configuration already contains the
         correct content but the system bypassed it; \emph{raise the fix to a
         higher-reliability artifact representation}.
\end{itemize}
Preferences: prefer REMOVE/REPLACE over ADD; cite the responsible chain
candidate in each \texttt{artifact\_issue}; make no tit-for-tat additions; and hunt
\emph{suppression} edges and \emph{lost-win} entries so a fix for one instance
does not re-break another.

\subsection{Single-Shot Baseline}
\label{app:reflect:pm}

\paragraph{Single-shot search.}
By \emph{single-shot search} we mean the standard mutate--evaluate--select loop:
mutate a parent into a candidate, evaluate the candidate once on a batch of
tasks (a few trials only to denoise its pass rate), and admit it to the
population only if $\textit{clears}(s_B(\mathrm{child}),s_B(\mathrm{parent}),\delta)$
for the benchmark- and configuration-dependent acceptance predicate---otherwise
discard or revert it and propose a fresh candidate next iteration. A rejected candidate is never
re-run on the same tasks to learn why it failed. This is the regime of prior
artifact and prompt optimizers (\S\ref{sec:related}) and the baseline the \mc
improves; Algorithm~\ref{alg:mainloop} with $K{=}0$ recovers it.

Theorems~\ref{thm:wall} and \ref{thm:rc} are generic conditional statements,
not a comparison of the recorded configurations. The first concerns any
accepted-update process whose proposal support excludes a condition; the second
adds realizability, non-regression, and a history-conditional positive-support
assumption. Neither assumption is established by the empirical comparisons.

\section{Experiment Protocol}
\label{app:expproto}

\paragraph{Data selection.}
We pick three benchmarks that span three classes of tunable artifacts
(\S\ref{sec:problem}): AppWorld (skill artifact configuration),
TerminalBench~2.1 (agent-harness artifact configuration), and MuSiQue
(retrieval-pipeline artifact configuration). Documented comparisons use harness
optimizers (AHE, Meta-Harness), human CLI agents (Codex, OpenCode), and the
un-optimized base on TerminalBench~2.1, and the hand-written default pipeline on
MuSiQue. AppWorld records are static evaluations of optimized skill configurations, not
documented optimizer runs or a method comparison.

\paragraph{Data, splits, and benchmark versions.}
\emph{AppWorld} (v0.2.0.dev0, commit \texttt{e9325d3}): we use the benchmark's
\texttt{train\_val\_split.json}, which merges the official \texttt{train} and
\texttt{dev} tasks into a \textbf{147}-task training pool from which
optimization minibatches are drawn. The validation set contains \textbf{200}
tasks sampled from the \texttt{test\_normal}/\texttt{test\_challenge} pools
(\textbf{100} each, seed $42$), stratified so that $\sim$$10\%$ are solvable
by the empty-skill baseline agent ($23$ pass / $177$ fail). Final static
evaluations are reported on the full test splits: \textbf{585} tasks
(\textbf{168} \texttt{test\_normal} and \textbf{417}
\texttt{test\_challenge}; Table~\ref{tab:appworld}), which therefore include
the $200$ validation tasks; validation scores are used only for optimization
guidance and climbing-speed curves, never reported as final results.
\texttt{max\_steps}$=50$ throughout.
\emph{MuSiQue}: \texttt{tools/prepare\_musique.py} converts the raw
\texttt{.jsonl} into BEIR-style files (\texttt{corpus.jsonl} with globally
unique \texttt{doc\_id}; \texttt{queries.jsonl} with gold from
\texttt{is\_supporting=True}). We use the 4-hop subset:
\textbf{300} train (from \texttt{musique\_full\_v1.0\_train}, 4-hop, first 300
with gold) / \textbf{100} val (from \texttt{musique\_full\_v1.0\_dev}, 4-hop,
100 for development) / \textbf{300} test (from the \texttt{dev} full set,
4-hop, excluding the 100 val ids; test and val have zero overlap)
(Table~\ref{tab:musique}).
\emph{TerminalBench~2.1}: \textbf{89} tasks; there is no separate val
set---all 89 tasks serve as both the training set for the optimizer and the
evaluation set. This follows the standard protocol used by prior work on
Terminal-Bench~2.0/2.1 (including AHE and Meta-Harness), which optimize
directly on the full task set and report pass rate over the same 89 tasks.
Because there is no held-out set, the reported TerminalBench results are
in-sample optimized-task rates rather than held-out scores
(Figure~\ref{fig:main-results}(b)).

\paragraph{Metrics.}
\emph{AppWorld:} strict pass/fail per task. Each task has multiple test
requirements (assertions on database state, answer correctness, side effects).
A task scores $1.0$ only if \emph{all} requirements pass; $0.0$ if any fails.
The test suite is deterministic (time frozen, DB state compared exactly).
\textbf{Task Goal Completion (TGC)} = average of per-task scores across all
tasks. \textbf{Scenario Goal Completion (SGC)} = for each scenario (a group of
related tasks sharing state), take the minimum score among its tasks (an entire
scenario passes only if \emph{every} task in it passes), then average across
scenarios. TGC and SGC are reported separately on the Normal and Challenge
splits (Table~\ref{tab:appworld}).
\emph{TerminalBench~2.1:} pass/fail (binary) per task; reported as pass rate
over all 89 tasks (Figure~\ref{fig:main-results}(b)).
\emph{MuSiQue:} nDCG@10 and Recall@10, both computed over the top-10 retrieved
documents per query (Table~\ref{tab:musique}). nDCG@10 (normalized
Discounted Cumulative Gain at cutoff 10) measures ranking quality: it rewards
placing golden documents higher in the list, with logarithmically diminishing
returns for lower ranks. Recall@10 measures coverage: the fraction of all
golden documents that appear in the top-10 retrieved results. Both metrics are
standard information-retrieval measures and are computed deterministically from
the ranked output of the retrieval pipeline.

The archived AppWorld final-result records provide TGC and SGC values over the
reported Normal and Challenge task sets. They do not establish whether the final
evaluator, optimization reward, minibatches, or validation gate used the same
interface, or how any optimized artifact configuration was produced; no stronger reproducibility
claim is made here.

\paragraph{Per-method launch parameters.}
The documented framework settings for TerminalBench~2.1 and MuSiQue
use a GLM-5.1 Pi coding-agent proposer; AppWorld proposer settings are
unavailable. The documented settings share
\texttt{max\_turns}$=200$, \texttt{timeout}$=3600$\,s,
\texttt{skip\_perfect\_score\_runs}$=$False, \texttt{last\_n\_analysis}$=10$,
\texttt{epochs}$=3$.

\emph{AppWorld.} The archived local records establish only the $168$
\texttt{test\_normal} and $417$ \texttt{test\_challenge} static final
evaluations of optimized skill configurations. No optimizer manifests exist;
optimization-run settings, system and proposer assignments, minibatches,
validation procedures, budgets, and baseline launch parameters are unavailable.
The labels 1-slot and 2-slot have no documented historical meaning. No AppWorld
protocol beyond those static records is claimed.
\emph{TerminalBench~2.1}:
\begin{itemize}[leftmargin=2em,itemsep=0.15em,topsep=0.1em]\raggedright
  \item \textsc{Mara Chain} (system GLM-5, system concurrency $4$, evaluator
        concurrency $4$, proposer concurrency $20$,
        \texttt{max\_candidate\_num}$=1$, \texttt{batch\_size}$=4$,
        budget $100$\,h wall-clock):
    \begin{itemize}[leftmargin=1.5em,itemsep=0.05em,topsep=0.05em]
      \item $+$\mc: \texttt{num\_proposals}$=1$,
            \texttt{max\_reflection\_iterations}$=5$.
      \item $-$\mc: \texttt{num\_proposals}$=1$,
            \texttt{max\_reflection\_iterations}$=0$.
    \end{itemize}
  \item AHE~\citep{lin2026ahe}: system model GLM-5, analyzer/improve model
        GLM-5.1 (max tokens $64000$, temperature $0.2$),
        max iterations $10$, target pass rate $0.95$, pass@$k$ with $k{=}2$,
        eval concurrency $4$, auto-rollback-harmful True, real NexAU
        evolve-agent enabled, budget $100$\,h wall-clock. AHE's method
        (AgentDebugger analyze $\to$ NexAU evolve-agent mutate $\to$ falsify
        with verdict $\to$ optional auto-rollback) is faithfully reproduced
        from the open-source repo
        (\url{https://github.com/china-qijizhifeng/agentic-harness-engineering})
        on TerminalBench~2.1, evolving the kira harness. The base model is GLM-5 (vs.\ the paper's
        GPT-5.4); absolute scores are not directly comparable to the paper,
        only same-environment comparisons are valid. For a fair same-environment
        comparison with \textsc{Mara Chain} (which does not use external
        knowledge), we \emph{disabled AHE's external-knowledge acquisition}---its
        retrieval of outside context during harness evolution---in our runs.
  \item Meta-Harness~\citep{lee2026metaharness}: reproduced from the open-source
        Meta-Harness repo
        (\url{https://github.com/stanford-iris-lab/meta-harness}), which we
        adapt to the kira harness and evaluate under local
        Docker on the Terminal-Bench~2.1. Launch parameters: system model GLM-5, proposer model GLM-5.1 run as a Claude Code
        coding agent, pass@$k$ with $k{=}2$ (two trials per task), eval concurrency $4$, and a
        total budget of $100$\,h wall-clock.
\end{itemize}
\emph{MuSiQue}:
\begin{itemize}[leftmargin=2em,itemsep=0.15em,topsep=0.1em]\raggedright
  \item \textsc{Mara Chain} (non-agentic pipeline; evaluator concurrency $8$,
        proposer concurrency $10$, \texttt{max\_candidate\_num}$=3$,
        budget $20$\,h wall-clock):
        \texttt{num\_proposals}$=2$ (2-slot), \texttt{batch\_size}$=6$,
        \texttt{min\_improvement\_per\_batch}$=2.0$; the reported optimized
        pipeline uses \texttt{max\_reflection\_iterations}$=3$.
        The pipeline's internal node models (BAAI/bge-m3 for dense recall,
        BAAI/bge-reranker-v2-m3 for cross-encoder rerank~\citep{chen2024bgem3}, GLM-5.2 for
        \texttt{query\_rewrite}) are part of the evolved pipeline
        (App.~\ref{app:pipelinespec}).
\end{itemize}

\section{AppWorld Score Versus Optimization Time}
\label{app:timecurve}

Figure~\ref{fig:appworld-time} plots the same AppWorld optimization
trajectories as Figure~\ref{fig:main-results}(a), but against effective
wall-clock time instead of optimization rollouts. The efficiency gap is not an
artifact of per-rollout cost differences: \mc reaches the $0.8$ validation
score in $11.9$\,h, about $3\times$ faster than GEPA ($35.8$\,h), while ACE
and SkillOpt-Lite do not reach $0.8$ at any point in the recorded runs and
plateau at $0.770$ and $0.665$, respectively. \mc also attains the highest
final score ($0.870$ versus GEPA's $0.805$), matching the rollout-based view.

\begin{figure}[h]
\centering
\includegraphics[width=0.62\linewidth]{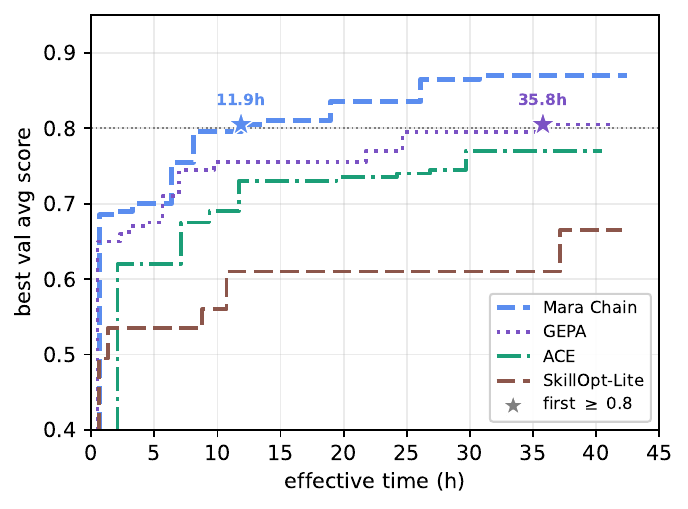}
\caption{AppWorld validation best score versus effective optimization time
(hours); the wall-clock companion of Figure~\ref{fig:main-results}(a). The
dotted line marks the $0.8$ validation score and $\bigstar$ annotates each
method's first recorded point at or above it. \mc reaches $0.8$ in $11.9$\,h,
about $3\times$ faster than GEPA ($35.8$\,h), and attains the highest final
score; ACE (final $0.770$) and SkillOpt-Lite ($0.665$) never reach $0.8$.}
\label{fig:appworld-time}
\end{figure}

\section{TerminalBench~2.1 Per-Task Reward}
\label{app:tb2-cross}

Table~\ref{tab:tb2-cross} lists the per-task reward ($r\in\{0,1\}$) on the
89-task set behind Figure~\ref{fig:main-results}(b). Column means reproduce
the reported pass rates: $51.7$, $49.4$, $51.7$, $71.9$, $57.3$, $51.7$,
$49.4\%$. Columns (left $\to$ right): Codex, OpenCode, Kira base,
\textsc{Mara Chain} (ours), $-$\mc ablation, AHE, Meta-Harness. All seven use
the same frozen GLM-5 system model.

\begin{small}
\setlength{\tabcolsep}{3pt}
\begin{longtable}{@{}>{\raggedright\arraybackslash\ttfamily}p{0.42\linewidth}ccccccc@{}}
\caption{TerminalBench~2.1 per-task reward ($1$ = task passes, $0$ = otherwise).} \\
\label{tab:tb2-cross} \\
\toprule
Task & \rotatebox{70}{Codex} & \rotatebox{70}{OpenCode} & \rotatebox{70}{Kira} & \rotatebox{70}{\textsc{Mara Chain}} & \rotatebox{70}{$-$\mc} & \rotatebox{70}{AHE} & \rotatebox{70}{Meta-Harness} \\
\midrule
\endfirsthead
\caption[]{TerminalBench~2.1 per-task reward (continued).} \\
\toprule
Task & \rotatebox{70}{Codex} & \rotatebox{70}{OpenCode} & \rotatebox{70}{Kira} & \rotatebox{70}{\textsc{Mara Chain}} & \rotatebox{70}{$-$\mc} & \rotatebox{70}{AHE} & \rotatebox{70}{Meta-Harness} \\
\midrule
\endhead
\midrule
\multicolumn{8}{r}{\footnotesize continued on next page} \\
\endfoot
\bottomrule
\endlastfoot
\texttt{bn-fit-modify} & 1 & 1 & 1 & 1 & 1 & 1 & 1 \\
\texttt{build-pmars} & 1 & 1 & 1 & 1 & 1 & 1 & 1 \\
\texttt{cobol-modernization} & 1 & 1 & 1 & 1 & 1 & 1 & 1 \\
\texttt{constraints-scheduling} & 1 & 1 & 1 & 1 & 1 & 1 & 1 \\
\texttt{count-dataset-tokens} & 1 & 1 & 1 & 1 & 1 & 1 & 1 \\
\texttt{crack-7z-hash} & 1 & 1 & 1 & 1 & 1 & 1 & 1 \\
\texttt{custom-memory-heap-crash} & 1 & 1 & 1 & 1 & 1 & 1 & 1 \\
\texttt{distribution-search} & 1 & 1 & 1 & 1 & 1 & 1 & 1 \\
\texttt{feal-differential-cryptanalysis} & 1 & 1 & 0 & 1 & 0 & 0 & 0 \\
\texttt{financial-document-processor} & 1 & 1 & 1 & 1 & 0 & 1 & 0 \\
\texttt{fix-code-vulnerability} & 1 & 1 & 1 & 1 & 1 & 1 & 0 \\
\texttt{fix-ocaml-gc} & 1 & 1 & 0 & 1 & 1 & 1 & 0 \\
\texttt{git-leak-recovery} & 1 & 1 & 1 & 1 & 1 & 1 & 1 \\
\texttt{headless-terminal} & 1 & 1 & 1 & 1 & 1 & 1 & 1 \\
\texttt{hf-model-inference} & 1 & 1 & 1 & 1 & 1 & 1 & 1 \\
\texttt{large-scale-text-editing} & 1 & 1 & 1 & 1 & 0 & 1 & 1 \\
\texttt{largest-eigenval} & 1 & 1 & 0 & 1 & 1 & 1 & 1 \\
\texttt{llm-inference-batching-scheduler} & 1 & 1 & 1 & 1 & 0 & 1 & 1 \\
\texttt{log-summary-date-ranges} & 1 & 1 & 1 & 1 & 1 & 1 & 0 \\
\texttt{mcmc-sampling-stan} & 1 & 1 & 1 & 1 & 1 & 1 & 1 \\
\texttt{merge-diff-arc-agi-task} & 1 & 1 & 1 & 1 & 0 & 0 & 1 \\
\texttt{modernize-scientific-stack} & 1 & 1 & 1 & 1 & 1 & 1 & 1 \\
\texttt{multi-source-data-merger} & 1 & 1 & 1 & 1 & 1 & 1 & 1 \\
\texttt{nginx-request-logging} & 1 & 1 & 1 & 1 & 1 & 1 & 1 \\
\texttt{openssl-selfsigned-cert} & 1 & 1 & 1 & 1 & 1 & 1 & 1 \\
\texttt{portfolio-optimization} & 1 & 1 & 0 & 1 & 1 & 1 & 0 \\
\texttt{prove-plus-comm} & 1 & 1 & 1 & 1 & 1 & 1 & 1 \\
\texttt{pytorch-model-cli} & 1 & 1 & 1 & 1 & 1 & 1 & 1 \\
\texttt{pytorch-model-recovery} & 1 & 1 & 1 & 1 & 1 & 1 & 1 \\
\texttt{regex-log} & 1 & 1 & 1 & 1 & 0 & 1 & 1 \\
\texttt{sparql-university} & 1 & 1 & 1 & 1 & 1 & 0 & 0 \\
\texttt{sqlite-db-truncate} & 1 & 1 & 1 & 1 & 1 & 1 & 1 \\
\texttt{sqlite-with-gcov} & 1 & 1 & 1 & 1 & 1 & 1 & 1 \\
\texttt{tune-mjcf} & 1 & 1 & 1 & 1 & 0 & 0 & 0 \\
\texttt{vulnerable-secret} & 1 & 1 & 1 & 1 & 1 & 1 & 1 \\
\texttt{cancel-async-tasks} & 1 & 0 & 1 & 1 & 1 & 0 & 0 \\
\texttt{code-from-image} & 1 & 0 & 0 & 1 & 0 & 1 & 1 \\
\texttt{compile-compcert} & 1 & 0 & 1 & 1 & 0 & 0 & 0 \\
\texttt{configure-git-webserver} & 1 & 0 & 0 & 1 & 1 & 1 & 1 \\
\texttt{fix-git} & 0 & 1 & 1 & 1 & 1 & 1 & 1 \\
\texttt{git-multibranch} & 1 & 0 & 1 & 1 & 1 & 1 & 1 \\
\texttt{overfull-hbox} & 1 & 0 & 1 & 1 & 1 & 1 & 1 \\
\texttt{password-recovery} & 0 & 1 & 1 & 1 & 1 & 1 & 1 \\
\texttt{path-tracing-reverse} & 0 & 1 & 1 & 1 & 0 & 0 & 0 \\
\texttt{pypi-server} & 0 & 1 & 1 & 1 & 1 & 1 & 1 \\
\texttt{qemu-alpine-ssh} & 0 & 1 & 0 & 1 & 1 & 1 & 1 \\
\texttt{qemu-startup} & 0 & 1 & 0 & 1 & 1 & 0 & 1 \\
\texttt{query-optimize} & 1 & 1 & 0 & 0 & 1 & 0 & 0 \\
\texttt{reshard-c4-data} & 1 & 0 & 1 & 1 & 0 & 0 & 1 \\
\texttt{sanitize-git-repo} & 1 & 0 & 0 & 1 & 0 & 0 & 0 \\
\texttt{torch-tensor-parallelism} & 1 & 0 & 0 & 1 & 1 & 0 & 0 \\
\texttt{winning-avg-corewars} & 0 & 1 & 0 & 1 & 1 & 0 & 0 \\
\texttt{adaptive-rejection-sampler} & 0 & 0 & 0 & 1 & 1 & 0 & 0 \\
\texttt{break-filter-js-from-html} & 0 & 0 & 0 & 1 & 1 & 1 & 1 \\
\texttt{build-pov-ray} & 0 & 1 & 1 & 0 & 1 & 0 & 0 \\
\texttt{circuit-fibsqrt} & 0 & 0 & 0 & 1 & 1 & 0 & 0 \\
\texttt{db-wal-recovery} & 0 & 0 & 0 & 1 & 1 & 0 & 0 \\
\texttt{extract-elf} & 0 & 0 & 0 & 1 & 0 & 1 & 1 \\
\texttt{install-windows-3.11} & 0 & 0 & 0 & 1 & 0 & 0 & 0 \\
\texttt{kv-store-grpc} & 0 & 0 & 1 & 1 & 1 & 1 & 1 \\
\texttt{mailman} & 0 & 0 & 1 & 1 & 1 & 1 & 1 \\
\texttt{model-extraction-relu-logits} & 0 & 0 & 0 & 1 & 0 & 0 & 0 \\
\texttt{mteb-retrieve} & 0 & 0 & 0 & 1 & 1 & 0 & 0 \\
\texttt{polyglot-c-py} & 0 & 0 & 1 & 1 & 1 & 1 & 1 \\
\texttt{protein-assembly} & 0 & 0 & 0 & 1 & 0 & 0 & 0 \\
\texttt{rstan-to-pystan} & 0 & 0 & 0 & 1 & 1 & 0 & 1 \\
\texttt{write-compressor} & 1 & 0 & 0 & 0 & 0 & 0 & 0 \\
\texttt{build-cython-ext} & 0 & 0 & 0 & 0 & 0 & 0 & 0 \\
\texttt{caffe-cifar-10} & 0 & 0 & 0 & 0 & 0 & 0 & 0 \\
\texttt{chess-best-move} & 0 & 0 & 0 & 0 & 0 & 0 & 0 \\
\texttt{dna-assembly} & 0 & 0 & 0 & 0 & 0 & 0 & 0 \\
\texttt{dna-insert} & 0 & 0 & 0 & 0 & 0 & 0 & 0 \\
\texttt{extract-moves-from-video} & 0 & 0 & 0 & 0 & 0 & 0 & 0 \\
\texttt{feal-linear-cryptanalysis} & 0 & 0 & 0 & 0 & 0 & 0 & 0 \\
\texttt{filter-js-from-html} & 0 & 0 & 0 & 0 & 0 & 1 & 0 \\
\texttt{gcode-to-text} & 0 & 0 & 0 & 0 & 0 & 0 & 0 \\
\texttt{gpt2-codegolf} & 0 & 0 & 0 & 0 & 0 & 0 & 0 \\
\texttt{make-doom-for-mips} & 0 & 0 & 0 & 0 & 0 & 0 & 0 \\
\texttt{make-mips-interpreter} & 0 & 0 & 0 & 0 & 0 & 0 & 0 \\
\texttt{mteb-leaderboard} & 0 & 0 & 0 & 0 & 0 & 0 & 0 \\
\texttt{path-tracing} & 0 & 0 & 0 & 0 & 0 & 0 & 0 \\
\texttt{polyglot-rust-c} & 0 & 0 & 1 & 0 & 1 & 1 & 0 \\
\texttt{raman-fitting} & 0 & 0 & 0 & 0 & 0 & 0 & 0 \\
\texttt{regex-chess} & 0 & 0 & 0 & 0 & 0 & 0 & 0 \\
\texttt{sam-cell-seg} & 0 & 0 & 0 & 0 & 0 & 0 & 0 \\
\texttt{schemelike-metacircular-eval} & 0 & 0 & 1 & 0 & 0 & 0 & 0 \\
\texttt{torch-pipeline-parallelism} & 0 & 0 & 0 & 0 & 0 & 0 & 1 \\
\texttt{train-fasttext} & 0 & 0 & 0 & 0 & 0 & 0 & 0 \\
\texttt{video-processing} & 0 & 0 & 0 & 0 & 0 & 0 & 0 \\
\end{longtable}
\end{small}

\section{MuSiQue Pipeline Artifact Configuration (Before/After)}
\label{app:pipelinespec}

The MuSiQue retrieval-pipeline artifact configuration is a JSON DAG of nodes, each with a type, per-node
parameters, and (for the rewrite node) a prompt file. \textsc{Mara Chain} evolves
this directory.

\paragraph{Default (baseline).}
{\scriptsize\begin{verbatim}
recall_bm25(k=10) -> recall_dense(k=10, bge_m3) ->
rerank_cross(top_m=100, bge-reranker-v2-m3) -> graph_rescore(a=0.3) ->
truncate(k=10)
\end{verbatim}}

\paragraph{Evolved.}
{\scriptsize\begin{verbatim}
query_rewrite(mode=llm, GLM-5.2, prompt_file=rewrite, T=0, max_tokens=128) ->
recall_bm25(k=30, use_rewritten=true) ->
recall_dense(k=30, bge_m3, use_rewritten=true, use_dual_query=true) ->
fuse_rrf(rrf_k=60) ->
graph_expand(top_n=40, min_entity_overlap=1, max_expand_per_group=10) ->
rerank_cross(top_m=30, bge-reranker-v2-m3, use_rewritten=false) ->
graph_rescore(a=0.2) -> truncate(k=30)
\end{verbatim}}

The proposer edited three kinds of assets: \emph{node code} (added
\texttt{query\_rewrite}, \texttt{fuse\_rrf}~\citep{cormack2009rrf},
\texttt{graph\_expand}; modified recall $k$ $10\!\to\!30$, rerank \texttt{top\_m}
$100\!\to\!30$, \texttt{graph\_rescore} $\alpha$ $0.3\!\to\!0.2$; rewired the
DAG), a \emph{prompt} (the rewrite node's synonym/term-expansion prompt), and
\emph{hyperparameters} (RRF \texttt{rrf\_k}$=60$, \texttt{graph\_expand}
\texttt{top\_n}$=40$ / \texttt{max\_expand\_per\_group}$=10$).

\paragraph{Ablation: the chain's contribution on MuSiQue.}
The ablation in Table~\ref{tab:ablation} (App.~\ref{app:framework}) isolates
the \mc's contribution on this retrieval-pipeline artifact configuration. At
the same 2-slot budget, disabling the chain ($-$\mc: single-shot search,
Mara-chain depth $0$) still improves over the hand-written default ($+0.070$
test nDCG@10), but trails the full setting by $0.034$ (test) and $0.048$ (val)
nDCG@10 and by $0.039$ (test) and $0.045$ (val) Recall@10. The chain therefore
accounts for roughly a third of the total test gain, mirroring the ablation
pattern on AppWorld and TerminalBench~2.1.

\section{Analogy to Classical Optimization}
\label{app:analogy}

The framework main loop of is a recognizable instance of classical iterative
optimization, with the artifact configuration in place of the parameter vector and the
\mc in place of \emph{momentum}---an operational, not decorative,
analogy. A classical optimizer carries a correction term across steps so the
next update builds on the last rather than restarting from zero, and the
\mc plays exactly this role: accumulating diagnoses and edits
down a single lineage so the next attempt starts from the residual the
previous one left, instead of discarding it the way single-shot search does
(Thms.~\ref{thm:wall}--\ref{thm:rc}). What is \emph{not} classical is the
substrate: the ``gradient'' is a rollout-derived edit attributed to a traced
failure, and the ``correction term'' is a diagnosis in natural language and
code, carried on a filesystem---which is why a coding agent, not a fixed
update rule, is the optimizer.

\begin{table}[h]
\centering
\small
\caption{The evolutionary-optimization analogy is \emph{operational rather than
decorative}: each row has a functional equivalent in the loop.}
\label{tab:analogy}
\begin{tabular}{ll}
\toprule
Known concept & framework functional equivalent \\
\midrule
parameter & artifact configuration (artifact directory tree) \\
search space & schema-instantiated artifact tree (not a fixed DSL) \\
gradient direction & rollout-derived edit direction (\texttt{ArtifactAction}) \\
learning rate & edit budget per epoch \\
optimizer & reflective proposer (two-phase coding agent) \\
momentum & \mc (correction accumulated across a chain of failures) \\
early stopping / checkpoint & candidate state machine + $(epoch,\,dataset\_index)$ resume \\
\bottomrule
\end{tabular}
\end{table}

\section{Persistent Failure Barrier Case Studies (Full)}
\label{app:cases}

This appendix presents four illustrative traces relevant to
\S\ref{app:cases:barriers}; they do not establish the assumptions of
Thms.~\ref{thm:wall}--\ref{thm:rc}. Because tasks run in batches (\texttt{batch\_size}$=4$),
a propose that earns reward on one task may be accepted even if another task
in the batch still scores $0$, so the other task's partial fix is
\emph{incidentally merged} into the lineage on someone else's reward. Such
merges are unstable---they occur by coattail, not because the failing task was
targeted---and fall far short of the targeted re-diagnosis that resolves the
persistent failure barriers below.

\subsection{Persistent failure barrier case studies}
\label{app:cases:barriers}

\takeaway{The recorded traces contain four depth-$4$ passing descendants after
zero-scoring ancestors; they are illustrative, not causal identification.}
On TerminalBench~2.1, the recorded $-$\mc configuration does not pass the
listed tasks while a recorded retained lineage does. Each listed lineage has
zero-scoring ancestors and a later passing descendant. We trace four cases
(App.~\ref{app:cases} for full traces): three TerminalBench persistent failure
barriers resolved at
depth~$4$---\texttt{sanitize-git-repo},
\texttt{model-extraction-relu-logits} (Table~\ref{tab:case-extract}), and
\texttt{protein-assembly} (Table~\ref{tab:case-protein})---and one AppWorld
scenario, \texttt{29caf6f} (Table~\ref{tab:case-appworld-chain},
Fig.~\ref{fig:case-appworld}). In all four, every depth-$<4$ ancestor scores
$0$ and only the depth-$4$ child passes; in three the breakthrough asset is
general and family-reusable (\texttt{protein-assembly} is cleared by a
protein-specific rule). On the AppWorld persistent failure barrier, $+$\mc is stable and
$-$\mc unstable (Fig.~\ref{fig:case-appworld}, App.~\ref{app:cases}).
The following illustration traces one refinement lineage end to end;
the other three persistent residuals are traced below.

\subsection{Without \mc: Persistent Residuals}
\label{app:cases:wo}

Theorem~\ref{thm:wall} applies only under its residual-completeness and
accepted-update assumptions. The recorded $-$\mc run retires $78$ of $104$
candidates, including candidates with partial corrections.

\paragraph{\texttt{sanitize-git-repo}.} More than ten propose batches analyze
this persistent failure barrier, each adding another partial prose fix (a length-constrained-pattern
rule, a no-truncation rule, a \texttt{grep -roh} procedure, a repo-wide-search
rule, a git-checkout caveat). Several are re-made---the same diagnosis recurs
in three separate batches. Yet every sanitize run scores $0$ ($9$ candidates).
The partials either merge for the wrong reason (other tasks in the batch pass
while sanitize stays $0$) or are discarded; in both cases the partial stays at
prose (never raised to a runnable script), so the buried token is never
surfaced. The partial correction is neither refined on the same evidence nor
transferred to a higher-reliability script, so the residual persists.

\paragraph{\texttt{model-extraction-relu-logits}.} Four propose entries each
re-add essentially the same partial---a prose ``black-box extraction'' rule.
The partial is always prose (no runnable extraction script), so all $9$ runs
score $0$. Independent proposals repeatedly recreate the same prose partial,
but do not retain its evidence long enough to construct the extraction algorithm
the task requires.

\paragraph{\texttt{protein-assembly}.} $8$ candidates, all score $0$ on
\texttt{test\_gblock}---the agent uses HIV-1 capsid instead of the FLAG tag, or
strips tags from the specified source. $-$\mc does diagnose specific
errors (``HIV-1 capsid instead of FLAG tag'', ``no gBlock reference''), but
these candidates are retired; none builds the full correction chain (use the
specified source verbatim, handle X-residues, verify the first amino acid,
check the fusion order). Because the candidates are retired independently, these
partial corrections are not accumulated into a complete sequence.

\paragraph{AppWorld \texttt{29caf6f}.} $-$\mc solves only
$0.16$/$0.21$/$0.21$ of runs across the scenario's three tasks, every failure
tripping the \emph{same} sub-condition---``the reply names a non-requested
movie''---reproduced verbatim $54\times$. The skill artifact configuration already tells the agent to
extract the requested director and filter by it, but the rule is a
bypassable prose/regex asset ($d<\theta$); the agent discards the director and
sends all $21$ movies $54\times$. The extract-and-filter rule is a useful but
incomplete low-reliability artifact; its execution evidence is discarded rather
than used to refine the candidate or strengthen the artifact representation.

\subsection{With \mc: convergence}
\label{app:cases:rc}

The \mc restructures search as a single lineage
$\kappa=(\sigma_0,\ldots,\sigma_T)$. Each step conditions on the preceding
residual, rollout analysis, and artifact-change history; it may also select a
higher-reliability artifact representation when the evidence indicates bypass.
Newly observable failures become the residual for the next refinement step.
Theorem~\ref{thm:rc} gives a bound only when its positive-support assumption is
made. In these recorded traces, four passing descendants occur at depth~$4$ after
zero-scoring ancestors. The illustration gives the
\texttt{sanitize-git-repo} workflow, Tables~\ref{tab:case-extract}--
\ref{tab:case-protein} give the remaining TerminalBench traces, and
Table~\ref{tab:case-appworld-chain} gives the AppWorld scenario.

\paragraph{\texttt{sanitize-git-repo}: a $4$-step refinement builds the script.}
Each node conditions on the preceding
residual; the artifact changes from prose to script; and the depth-$4$ ``print
only the matched token'' edit corrects the script's own output behavior, which
became observable only after the script was introduced. Retaining each residual
allows the three failure modes to be addressed sequentially rather than
rediscovered independently.

\paragraph{\texttt{model-extraction}: a $4$-deep bloodline rewrites the
algorithm (Table~\ref{tab:case-extract}).} Winner score $1.0$, $20/20$ rows
match, cosine $>0.99$; every ancestor scores $0$. The artifact evolves from
prose to reference material and then to a script. The two script fixes address
distinct retained residuals---``too slow'' (execution) and ``over-counts''
(algorithmic correctness)---and the final rewrite corrects behavior observable
only after a runnable script exists.

\begin{table}[h]
\centering\small
\caption{Model extraction: one bloodline, depth $0\!\to\!4$, every ancestor
scoring~$0$. The artifact evolves from prose to reference material and then to
a script; the depth-$3$/$4$ algorithm rewrite corrects behavior exposed only
after the script is runnable.}
\label{tab:case-extract}
\setlength{\tabcolsep}{3pt}
\begin{tabularx}{\linewidth}{@{}r l l Y Y@{}}
\toprule
rd & candidate & form & asset edited & residual $\delta$ (drives next) \\
\midrule
0 & \texttt{f95e6890cfa1} & (start) & (inherits parent, no task-specific edit) & agent \texttt{cat}s \texttt{forward.py} (not black-box) \\
0 & \texttt{dd5d0e8a3382} & prose+ref & \texttt{SKILL} black-box rule + extraction ref & evaluator \texttt{np.random.seed} pollution \\
1 & \texttt{38899c80b3fe} & ext+ref & random-state isolation & correct method, no runnable impl \\
2 & \texttt{fe8a3fbab6c9} & ref+scr & \texttt{extract-relu-weights.py} v0 (freq.\ clustering) & too slow; agent abandons \\
3 & \texttt{d87d45d2e963} & ref+scr & efficiency: $2000\to800$ samples & over-counts: $47$ vs $20$ neurons \\
4 & \texttt{38b2f87817e4}\,\checkmark & ext+ref+scr & algorithm rewrite: line-scan & \textbf{PASS} \\
\bottomrule
\end{tabularx}
\end{table}

\paragraph{\texttt{protein-assembly}: a $4$-deep bloodline corrects the fusion
sequence (Table~\ref{tab:case-protein}).} It inherits the model-extraction
winner's artifact configuration, then goes its own depth-$4$ route on the protein
persistent failure barrier. The barrier is resolved by a \emph{protein-specific} X-residue rule---the one genuinely
per-task patch---though the same winning node also banks a general
cross-compiler that serves other tasks.

\begin{table}[h]
\centering\small
\caption{Protein assembly: one bloodline, depth $0\!\to\!4$, every ancestor
scoring~$0$. $\dagger$ marks a protein-specific fix and $^{\text{gen}}$ a
general (reusable) fix.}
\label{tab:case-protein}
\setlength{\tabcolsep}{3pt}
\begin{tabularx}{\linewidth}{@{}r l l Y Y@{}}
\toprule
rd & candidate & form & asset edited & residual $\delta$ (drives next) \\
\midrule
0 & \texttt{899d9d350389} & ext+prose & FC fix in \texttt{kira\_agent}$^{\text{gen}}$ & uses HIV-1 capsid instead of FLAG tag \\
0 & \texttt{52aa7b868cdc} & ext+prose & terminal spirals in \texttt{kira\_agent}$^{\text{gen}}$ & wrong source, strips tags from data \\
1 & \texttt{51e925ee9e77} & ext+prose & data-source checklist in \texttt{kira\_agent}$^\dagger$ + SKILL rule$^\dagger$ & assembly still wrong \\
2 & \texttt{216373b50662} & ext+prose & verbatim / His-tag rule in SKILL$^\dagger$ + runtime mod$^{\text{gen}}$ & X-residue handling wrong \\
3 & \texttt{c0f841ac9754} & ext+prose+scr & data-source rule in SKILL$^\dagger$ + tool alias$^\text{gen}$ + sanitizer$^\text{gen}$ & length offset, first amino acid \\
4 & \texttt{fddaf3b3e11c}\,\checkmark & ext+prose+scr & X-residue rule in SKILL$^\dagger$ + cross-compiler$^\text{gen}$ & \textbf{PASS} \\
\bottomrule
\end{tabularx}
\end{table}

\paragraph{AppWorld \texttt{29caf6f}: the chain hardens the scenario's shared
\texttt{SKILL.md} (Table~\ref{tab:case-appworld-chain},
Fig.~\ref{fig:case-appworld}).} Early fixes (gen~$3$--$7$) already make
$+$\mc's \texttt{\_2}/\texttt{\_3} pass rate far exceed $-$\mc's $0.21$
($0.60$--$0.75$ at gen~$2$--$3$), while $-$\mc stays at $0.21$
throughout---the early fixes, not just the late chain, are what makes
$+$\mc stable. The late chain (gen~$15$--$16$) adds the
requester-retrieval gate + director isolation (peak val $0.855$). Because the
three tasks share one \texttt{SKILL.md}, the hardening fixes
\texttt{\_2}/\texttt{\_3} \emph{and} stabilizes \texttt{\_1}.
Figure~\ref{fig:case-appworld} shows the pass/fail pattern: $+$\mc is
stable, $-$\mc unstable.

\begin{table}[h]
\centering\small
\caption{AppWorld \texttt{29caf6f}: key RC fixes on the shared \texttt{SKILL.md}
that harden the scenario, in generational order. Early fixes (gen~$3$--$7$)
already make WITH's pass rate far exceed $-$\mc's $0.21$; the late chain
(gen~$15$--$16$) adds the requester-retrieval gate + director isolation.}
\label{tab:case-appworld-chain}
\setlength{\tabcolsep}{4pt}
\begin{tabularx}{\linewidth}{@{}r l l Y l@{}}
\toprule
gen & rd & node & edit & solved \\
\midrule
3 & 1 & \texttt{39102e8295c9} & fix login rule, pagination order, substring collisions & \texttt{\_3} \\
4 & 1 & \texttt{ec65a28a8e6e} & step-$6$ banking: paginate all contacts + client-side filter (prerequisite; seeds checkpoint) & locates msg \\
4 & 2 & \texttt{f66bfcd1ab39} & prevent context overflow from exhaustive message pagination & \texttt{\_2},\,\texttt{\_3} \\
7 & 0 & \texttt{9ccdf0fbfefe} & add \texttt{pagination\_helper.py} + prevent restart loops ($+$\mc-only checkpoint) & sustains \\
14 & 1 & \texttt{a810bbf06036} & fix fragile parsing (line-by-line extraction) & \texttt{\_2},\,\texttt{\_3} \\
15 & 1 & \texttt{6ebfc3cbb94b} & enforce requester-message retrieval gate (root cause: agent never retrieved the request) & \texttt{\_2} \\
16 & 2 & \texttt{35a92e7e209d}\,\checkmark & isolate requested director in a variable; filter by it (peak val $0.855$) & all three \\
\bottomrule
\end{tabularx}
\end{table}

\begin{figure}[h]
\centering
\includegraphics[width=\linewidth]{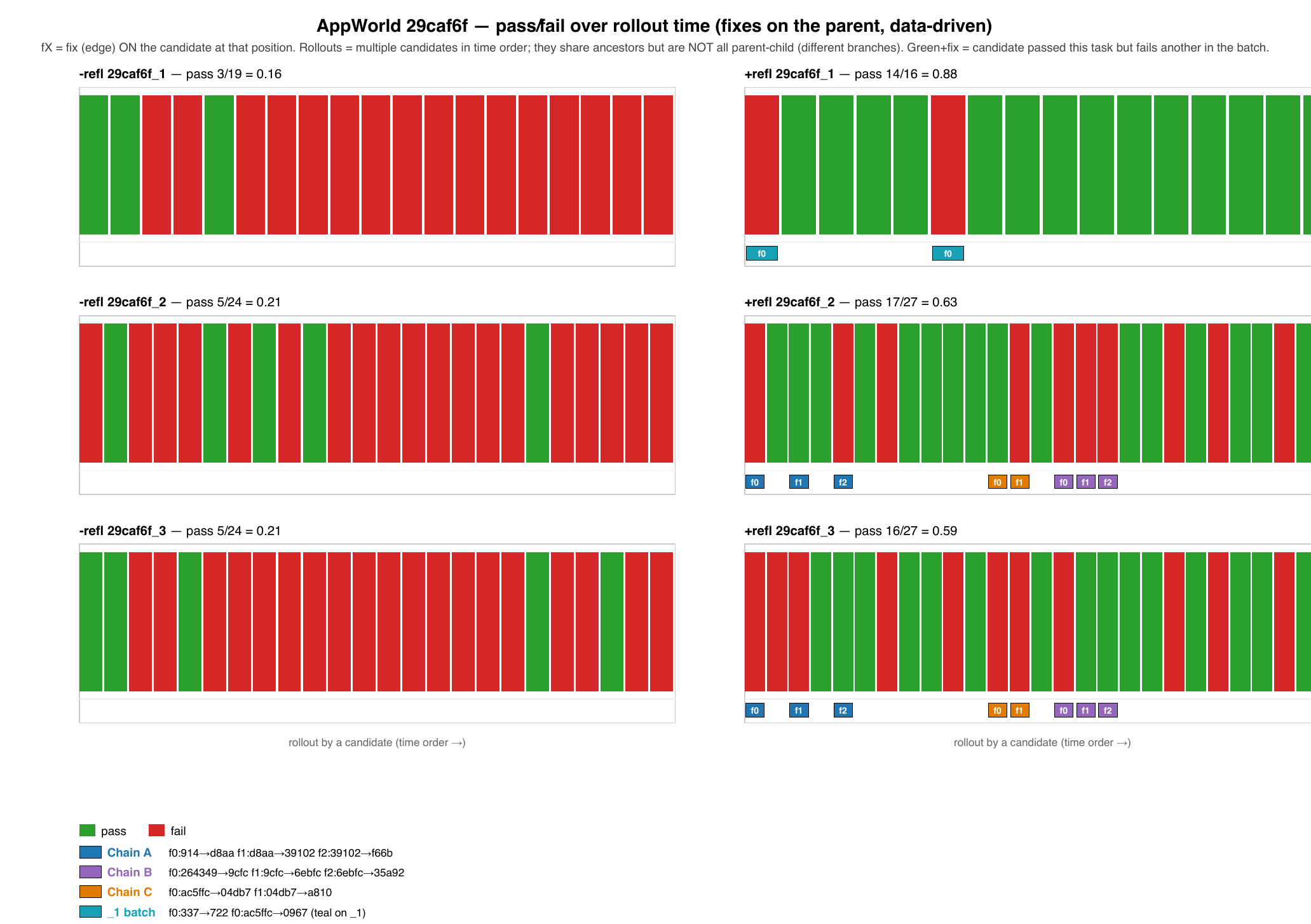}
\caption{AppWorld \texttt{29caf6f}: pass/fail of each task's in-run rollouts
during the $\pm$\mc ablation (green=pass, red=fail; left-to-right in
rollout time). $-$\mc (left) is unstable; $+$\mc (right) is
stable. Colored squares mark fixes (first rollout only); same color = same
chain. $f0/f1/f2$ = a fix producing an rd-$0/1/2$ candidate.}
\label{fig:case-appworld}
\end{figure}

\paragraph{Generality.}
Three of the four breakthrough assets are general and family-reusable:
\texttt{search-filtered.sh} takes any prefix set and prints only matched
tokens; \texttt{extract-relu-weights.py} auto-detects input dim and scans
random 1-D lines for any black-box ReLU net; the AppWorld extract-and-filter
serves any ``reply with items matching a requested attribute'' task.
\texttt{protein-assembly} is the exception (a protein-specific rule).

\takeaway{Four persistent residuals, one mechanism.}
Under the assumptions of Theorem~\ref{thm:wall}, unsupported accepted updates
preserve a residual. These traces show retained rollout evidence and artifact
history across later edits, but do not establish that retention alone caused the
observed passes or that budget played no role.

\end{document}